\documentclass[letterpaper, 10 pt, conference]{ieeeconf}  

\IEEEoverridecommandlockouts                              

\usepackage{amsthm}   
\usepackage{amsmath} 
\usepackage{mathtools} 
\usepackage{amssymb}  

\usepackage{multirow}
\usepackage{booktabs} 
\usepackage{graphicx} 
\let\labelindent\relax
\usepackage{enumitem} 
\usepackage{algorithm}
\usepackage{algorithmic}
\usepackage{url}

\usepackage{tikz}
\usetikzlibrary{arrows.meta,calc,positioning,decorations.pathreplacing}
\newtheorem{assumption}{Assumption}
\newtheorem{remark}{Remark}
\newtheorem{proposition}{Proposition}
\newtheorem{corollary}{Corollary}
\newtheorem{theorem}{Theorem}
\newtheorem{lemma}{Lemma}
\newtheorem{definition}{Definition}

\title{\LARGE \bf
Finite-Sample Analysis of Elimination in Active Hypothesis Testing
}

\author{
Ziyuan Lin,
Hoang Ngoc Nguyen,
Jie Xu,
Ivan Ruchkin%
\thanks{This work has been submitted to the IEEE 
for possible publication. Copyright may be 
transferred without notice, after which this 
version may no longer be accessible.}%
\thanks{This work was supported by the NSF Grants CNS 2440920 and EPCN 2512169. Any opinions, findings, conclusions, or recommendations expressed in this material are those of the authors and do not necessarily reflect the views of the National Science Foundation (NSF) or the US Government.}
\thanks{The authors are with the Department of Electrical and Computer Engineering,
University of Florida, Gainesville, FL 32611, USA
{\tt\small \{ziyuanlin, ngochoangnguyen, jie.xu\}@ufl.edu, iruchkin@ece.ufl.edu}}%
}

\begin{document}

\maketitle
\thispagestyle{empty}
\pagestyle{empty}

\begin{abstract}
A fixed-confidence, finite-sample problem of active hypothesis testing arises in many safety-critical applications. Situation in the context of sequential hypothesis testing, this paper studies the effect of hypothesis elimination on the stopping time. We introduce an elimination-augmented Track-and-Stop algorithm, in which champion-specific active-opponent sets are progressively pruned, and sensing effort is reallocated toward the surviving alternatives. Our analysis derives a non-asymptotic upper bound on the expected stopping time. The gain in finite-sample from elimination appears on the scale of the non-leading term, resulting from tighter tracking and concentration constants on the reduced hypothesis set. Furthermore, we introduce an aggressiveness parameter to modulate the trade-off between faster elimination and weaker confidence guarantee. An experimental study on synthetic Gaussian instances confirms the theoretical predictions.
\end{abstract}

\section{INTRODUCTION}

Active hypothesis testing studies a decision-maker that sequentially selects sensing actions to identify an unknown state of the environment as quickly as possible, subject to a prescribed reliability constraint~\cite{c1,c2}. The problem arises naturally in safety-critical control applications where both detection speed and correctness directly govern system safety: fault isolation in networked systems~\cite{c17}, target recognition with active sensors for autonomous vehicles~\cite{c18}, and anomaly detection in industrial processes~\cite{c14,c15}. 

\looseness=-1
The theoretical foundation of active sequential hypothesis testing traces back to Chernoff~\cite{c3}, who formulated the problem as minimizing a Bayesian risk combining expected stopping time and misclassification penalty, proposing an adaptive strategy to accelerate discrimination among hypotheses. Naghshvar and Javidi~\cite{c4} extended this Bayesian framework to $M$-ary hypothesis testing, bounding the minimum achievable Bayesian risk that trades off expected sample size against misclassification cost. Operating instead under frequentist risk constraints, Nitinawarat et al.~\cite{c5} proved asymptotic optimality of a modified version of the Chernoff test --- which selects action to maximize the difference between the current best hypothesis and its hardest opponent --- under per-hypothesis error control, and established a separation principle between maximum-likelihood estimation and optimal action selection. These works demonstrated the value of active over passive testing, but their guarantees are asymptotic: they characterize the optimal information acquisition rate as error chance $\delta\to 0$ without quantifying the stopping time at any fixed $\delta$. In safety-critical applications where $\delta$ is a hard design constraint, this motivates a \emph{fixed-confidence} formulation: given $\delta$, minimize the expected stopping time while guaranteeing misclassification probability at most $\delta$.

For the fixed-confidence setting, the most relevant algorithmic precedent comes from the pure-exploration bandit literature. The learner's task in active sensing --- choosing which action to observe to identify the true hypothesis --- has the same max-min information structure as bandit pure exploration, with best-arm and odd-arm identification as special cases~\cite{c6}. Relevant to us, Kaufmann, Cappé, and Garivier~\cite{c7} characterized the tight information-theoretic lower bounds on stopping time for best-arm identification. Building on this characterization, Garivier and Kaufmann~\cite{c8} realized these bounds with the \textit{Track-and-Stop (TaS)} algorithm, which combines three ingredients: an oracle allocation that solves the max-min arm allocation problem, a tracking rule that forces the empirical sampling frequencies to follow this target allocation over time, and a Chernoff-style stopping rule based on generalized likelihood ratios. Thus, we borrow this asymptotically sampling and stopping mechanism from the bandit literature. However, TaS always computes its sampling target against the \emph{full} set of alternative hypotheses, even when strong statistical evidence argues against many of them. Hence, active effort continues to be governed by the hardest pair to discriminate in the full set, rather than focusing on the hypotheses that remain genuinely difficult to rule out.

This leads to a natural insight: eliminate clearly implausible hypotheses during the test and redirect sensing effort toward the remaining active alternatives. Recent results establish that elimination is asymptotically optimal and empirically beneficial: Tirinzoni and Degenne~\cite{c9} showed that elimination can be layered on top of adaptive algorithms (e.g., TaS) without degrading sample complexity guarantees, while Vershinin et al.~\cite{c10} demonstrated empirical reductions in detection delay via progressive hypothesis pruning. 

However, neither of these works provides a precise characterization of \emph{finite-sample} acceleration: once the active set shrinks, the oracle allocation is recomputed against a smaller set of alternatives, raising the effective information rate. It is not currently known how much it is raised and under what conditions.  This gap is especially significant in multi-hypothesis active sensing, where elimination changes the sampling target mid-test, invalidating the standard asymptotic analysis that assumes a fixed oracle allocation throughout.

In this paper, we fill that gap for fixed-confidence active multi-hypothesis sensing. We study an elimination-based TaS procedure that sequentially prunes unlikely hypotheses and reallocates sensing effort over the remaining active set. We also introduce a tuning parameter $\alpha \in (0,1]$ that controls the aggressiveness of elimination: smaller $\alpha$ triggers earlier pruning, whereas $\alpha=1$ delays elimination until the nominal certification threshold is met. 

This paper's contributions are as follows:
\begin{itemize}
    \looseness=-1
    \item \textbf{Finite-sample analysis of elimination-based TaS.} We establish a finite-sample upper bound on the expected stopping time of a Track-and-Stop algorithm with sequential hypothesis elimination. The bound makes explicit how shrinking the active set changes the effective allocation problem.
    
    \item \textbf{Analytical characterization of the time-error trade-off:}  The aggressiveness parameter $\alpha$ induces an explicit trade-off between stopping time and elimination error. We characterize analytically how the leading term of the stopping time bound varies with $\alpha$, and illustrate the resulting speed-reliability curve through simulation.

   \item \textbf{Empirical validation:} We experimentally validate the main theoretical predictions: the instance-dependent stopping-time reduction from active-set shrinkage, the speed-accuracy trade-off induced by $\alpha$, and the stage-wise acceleration mechanism underlying the observed gains through three targeted experiments.
    
\end{itemize}

\section{PRELIMINARIES}
\label{sec:background}
\subsection{Active Multi-Hypothesis Sensing Model}
Let $\mathcal{H} = \{1,2,\dots,K\}$ be a finite set of mutually exclusive hypotheses with unknown true state 
$h^* \in \mathcal{H}$, and let $\mathcal{A}$ be a finite set of sensing actions. Observations take values in a measurable space $\mathcal{O}$, which can be either discrete or continuous, i.e., a finite-dimensional Euclidean space.

At each time $t \geq 1$, the controller selects $A_t \in \mathcal{A}$ based on the history $(A_1, O_1, \dots, A_{t-1}, O_{t-1})$ and receives an independent sample $O_t$ from the observation distribution $P_{h^*}^{A_t}.$ We write $\mathbb{P}_{h^*}(\cdot)$ and $\mathbb{E}_{h^*}[\cdot]$ for the probability and expectation induced by true hypothesis $h^*$ and the sequential policy over the trajectory  $(A_1, O_1, A_2, O_2, \dots).$ And we write $N_a(t) := \sum_{s=1}^{t} \mathbf{1}\{A_s = a\}$ for the number of times that action $a$ has been selected up to time $t$.

For each action $a\in\mathcal{A}$, we assume that the observation distributions $\{P_h^a : h\in\mathcal{H}\}$ are absolutely continuous with respect to a common distribution $\mu_a$ on $\mathcal{O}$. Consequently, for every $h\in\mathcal{H}$ there exists a density $p_a(o\mid h)$ such that
\[
P_h^a(B) = \int_B p_a(o\mid h)\,d\mu_a(o)
\]
for every measurable set $B\subseteq\mathcal{O}$. Here, $p_a(o\mid h)$ is a probability density function when $\mathcal{O}$ is continuous and a probability mass function when $\mathcal{O}$ is discrete.

We measure the statistical distinguishability between hypotheses $h$ and $g$ under action $a$ by the action-wise KL divergence~\cite{c16}: for $h,g \in \mathcal H$ and $a \in \mathcal A$,
\[
d_a(h,g) := \mathrm{KL}(P_h^a \| P_g^a)
= \mathbb{E}_{O \sim P_h^a}\!\left[
\log\frac{p_a(O \mid h)}{p_a(O \mid g)}\right],
\]
Intuitively, $d_a(h,g)$ quantifies how much information a single observation under $a$ provides to distinguish $h$ from $g$.
\begin{assumption}[Identifiability]
\label{ass:identifiability}
For every $h \neq g$ in $\mathcal{H}$, there exists an action $a \in \mathcal{A}$ such that $d_a(h,g) > 0$.
\end{assumption}
This assumption ensures that any two distinct hypotheses can be statistically distinguished by at least one 
sensing action, which is necessary for exact identification.

\subsection{Policy and Objective}

\looseness=-1
To identify the true hypothesis $h^*$,  we seek sequential policies for active multi-hypothesis sensing, under the fixed-confidence criterion, where a policy $\pi = (\Phi, s, r)$  consists of a sampling rule, a stopping time, and a recommendation rule: 

\begin{itemize}
\item The \emph{sampling rule} $\Phi = \{\Phi_t\}_{t\ge 1}$, where $\Phi_t : (\mathcal{A}\times\mathcal{O})^{t-1}\to\mathcal{A}$ determines the action $A_t = \Phi_t(A_1,O_1,\dots,A_{t-1},O_{t-1})$ for each round based on past observations.
\end{itemize}
\begin{itemize}
\item The \emph{stopping rule}  $s = \{s_t\}_{t\ge 1}$, where
$s_t:(\mathcal{A}\times\mathcal{O})^t\to\{0,1\}$, inducing the \textit{stopping time} $$\tau := \inf\{t\ge 1: s_t(A_1,O_1,\dots,A_t,O_t)=1\}.$$ We consider such stopping rules that $\mathbb{P}_{h^*}(\tau < \infty) = 1$ for all 
$h^* \in \mathcal{H}$.
\end{itemize}
\begin{itemize}
\item  The recommendation rule $r$ that outputs the \emph{recommendation} $\hat{h}_{\tau} \in 
\mathcal{H}$ upon stopping: $$\hat{h} := r(A_1, O_1, \dots, A_\tau, O_\tau).$$
\end{itemize}

\begin{definition}[$\delta$-PAC policy]
\label{def:deltapac}
Fix $\delta\in(0,1)$ and an arbitrary true $h^*\in\mathcal H$. A policy $\pi$ is \emph{$\delta$ - Probably Approximately Correct (PAC)} if
\[
\mathbb{P}_{h^*}(\hat{h} \neq h^*) \leq \delta,
\quad \forall\, h^* \in \mathcal{H}.
\]
\end{definition}
\noindent\textbf{Objective.} Given $\delta\in(0,1)$, design a $\delta$-PAC policy minimizing the expected stopping time:
\[
\min_{\pi}\ \mathbb{E}_{h^*}[\tau]
\quad\text{s.t.}\quad
\mathbb{P}_{h^*}(\hat{h}\neq h^*)\leq\delta,
\quad\forall\,h^*\in\mathcal{H}.
\]

\subsection{Oracle Benchmark}
\label{subsec:oracle}
The fixed-confidence objective naturally induces an information-allocation problem. The efficiency of any sequential procedure is governed by the speed with which its detection strategy 
accumulates information against the hardest alternative in $\mathcal{H}\setminus\{h^*\}$. 

Consider the simplex over actions $\Delta(\mathcal{A}) := \{ w \in [0,1]^{|\mathcal{A}|} : \sum_{a} w_a = 1 \}$. For a candidate hypothesis $\hat{h}\in\mathcal{H}$ and a nonempty competing set $S\subseteq\mathcal{H}\setminus\{\hat{h}\}$, the \emph{worst-case information rate} under allocation $w\in\Delta(\mathcal A)$ :
\[
f_S(w;h):= \min_{g \in S} \sum_{a \in \mathcal{A}} w_a\, d_a(h,g),
\]
where $w_a$ is the weight in the allocation $w$ for action $a$. This leads us to an \textit{oracle allocation of actions} $w^*(h;S):= \arg\max_{w \in \Delta(\mathcal{A})} f_S(w)$ that maximizes the worst-case information rate and the corresponding \emph{oracle information rate} is $D^*(h;S) := f_S(w^*(h;S);h) = \max_{w \in \Delta(\mathcal{A})} \min_{g \in S} \sum_{a} w_a\, d_a(h,g).$ The oracle rate $D^*(h;S)$ determines the fundamental speed limit for $\delta$-PAC identification. Any $\delta$-PAC policy must accumulate $\log(1/\delta)$ units of evidence against every alternative, and the fastest achievable rate for every step is $D^*(h;S)$. Therefore, the expected stopping time must satisfy $\mathbb E_{h^*}[\tau]\ge\frac{1}{D^*(h;S)}\log\frac{1}{\delta}.$ A similar lower bound was provided in the existing work~\cite{c6}.

In Section~\ref{sec:Proof}, we show how elimination improves the information rate by extending the oracle allocation to a shrinking active hypothesis set.

\section{Track-and-Stop with Elimination}
\label{sec:Proof}
This section introduces the \emph{Elimination-based Track-and-Stop} (Elim-TaS). We first review the TaS baseline~\cite{c6}, then describe the elimination-augmented procedure.
\subsection{Track-and-Stop-Type Baseline}
\label{subsec:baseline}

If the true hypothesis $h^*$ were known in advance, one would allocate action effort according to the oracle allocation $w^*(h^*;\mathcal H\setminus\{h^*\})$. Since $h^*$ is unknown, the TaS strategy~\cite{c6}, originally proposed for best-arm identification, replaces it online by the current maximum-likelihood estimate and tracks the corresponding oracle allocation. We specialize this approach to the active multi-hypothesis sensing setting.

Define the cumulative log-likelihood of hypothesis $h \in \mathcal{H}$ at time $t$ as $
L_t(h) := \sum_{k=1}^{t} \log p_{A_k}(O_k \mid h)$, $L_0(h) = 0.$ Let $\hat{h}(t) \in \arg\max_{h \in \mathcal{H}} L_t(h)$ denote the current maximum-likelihood champion, with ties broken arbitrarily. 

Formally, the sampling rule uses \emph{C-Tracking}~\cite{c7}: at each round $t$ the algorithm maintains a \emph{cumulative target} $W_a^{\mathrm{tar}}(t):= \sum_{s=1}^{t} w^*_a(\hat{h}(s);\mathcal H\setminus\{\hat{h}(s)\})$, where $a\in\mathcal{A},$ which records how many times action $a$ should ideally have been selected up to time $t$ under the current champion sequence. The next action is then chosen to reduce the largest under-representation relative to this target: $A_{t+1}\in\arg\max_{a\in\mathcal{A}}\Bigl(W_a^{\mathrm{tar}}(t)-N_a(t)\Bigr).$ This rule implicitly provides forced exploration: every action is selected at least $\Omega(\sqrt{t})$~\cite{c10} times by round~$t$, ensuring that no action is permanently neglected. To decide when accumulated evidence is sufficient for identification, define the pairwise 
log-likelihood ratio $Z_t(h,g)$ for $ h,g \in \mathcal{H}$ as
\[
Z_t(h,g) := L_t(h) - L_t(g) 
= \sum_{k=1}^{t} \log\frac{p_{A_k}(O_k \mid h)}
{p_{A_k}(O_k \mid g)}.
\]
The baseline stopping time is based on thresholding the ratio:
\[
\tau := \inf\left\{ t \geq 1 :
\min_{g \neq \hat{h}(t)} Z_t(\hat{h}(t), g)
\geq \beta_{\mathrm{stop}}(t,\delta) \right\},
\]
and the final recommendation is $\hat{h}(\tau)$. The specific form of $\beta_{\mathrm{stop}}(t,\delta)$ is given in the next subsection.

\subsection{Sequential Adaptive Elimination Algorithm}
\label{subsec:adaptive_elimination}
We now describe the elimination-augmented procedure studied in this paper. Compared to the baseline Track-and-Stop, the key modification is that the sampling target is no longer computed against the full alternative hypotheses set, but against a dynamically shrinking \emph{active-opponents set}.

\begin{definition}[Thresholds]
Let $\gamma(t) := b\log t + c$ for fixed constants $b,c>0$ depending only on $|\mathcal{H}|$ and $|\mathcal{A}|$. Define the \emph{stopping threshold} and \emph{elimination threshold} by
\[
\beta_{\mathrm{stop}}(t,\delta)
:= \log\frac{1}{\delta} + \gamma(t),
\quad
\beta_{\mathrm{elim}}(t,\delta)
:= \alpha\log\frac{1}{\delta} + \gamma(t),
\]
where $\alpha\in(0,1]$ controls elimination aggressiveness.
\end{definition}

\begin{definition}[Active-opponent set]
For each candidate $i\in\mathcal{H}$, initialize $G_0(i)=\mathcal{H}\setminus\{i\}$. For $t\ge 0$,
\[
\begin{aligned}
G_t(i)
:= \Bigl\{\, g\in\mathcal{H}\setminus\{i\}\;:\; 
&\forall s\le t, \text{with}\ \hat h(s)=i\ \\
&\ Z_s(i,g) < \beta_{\mathrm{elim}}(s,\delta)
\,\Bigr\}.
\end{aligned}
\]
\end{definition}
\(G_t(i)\) contains exactly those opponents that have not yet been eliminated relative to candidate \(i\) by time \(t\). This construction therefore is candidate-specific: evidence that one candidate rules out an opponent does not imply that the same opponent can be discarded for every other candidate.

At the beginning of the round \(t+1\), the algorithm uses the current maximum-likelihood champion \(\hat h(t)\) and tracks the allocation of active-set-based oracles $u(t) := w^*\bigl(\hat{h}(t); G_t(\hat{h}(t))\bigr)$ using C-Tracking. Thus, compared with the baseline, the only change in the sampling rule is that the full alternative set \(\mathcal H\setminus\{\hat h_t\}\) is replaced by the current active-opponents set \(G_t(\hat h_t)\). Hypothesis $g \in G_t(\hat{h}(t))$ is eliminated in round $t$ when the log-likelihood ratio of the current champion against $g$ exceeds the elimination threshold. The active-opponent set of the current champion is updated:
\begin{equation*}
\begin{aligned}
G_{t+1}(i)
:= \begin{cases}
G_t(i)\setminus\Bigl\{\,g\in G_t(i):\\
\qquad Z_t(i,g)\ge \beta_{\mathrm{elim}}(t+1,\delta)\,\Bigr\},
& i=\hat h(t),\\[0.5ex]
G_t(i), & i\neq \hat h(t).
\end{cases}
\end{aligned}
\end{equation*}
Over time, the active-opponent set shrinks and eventually becomes empty. At that point, the algorithm terminates and outputs the current champion as the final recommendation.

The parameter $\alpha$ controls the trade-off between speed and accuracy. When $\alpha = 1$, the elimination threshold $\beta_{\mathrm{elim}}(t,\delta)$ coincides with the stopping threshold $\beta_{\mathrm{stop}}(t,\delta)$ used in the standard TaS strategy; in effect, the algorithm is $\delta$-PAC (Theorem~\ref{thm:delta-pac-correctness}) and performs sequential elimination using the same criterion as stopping. When $\alpha < 1$, we have $\beta_{\mathrm{elim}}(t,\delta) < \beta_{\mathrm{stop}}(t,\delta)$, which defines the \emph{aggressive-elimination regime}. In this regime, hypotheses are discarded more quickly, potentially reducing the stopping time, but at the cost of a higher risk of elimination errors and thus a weaker final classification guarantee. 

\begin{algorithm}[h]
\caption{Track-and-Stop with Adaptive Elimination}
\label{alg:tas-elim}
\begin{algorithmic}[1]
\REQUIRE Hypothesis set $\mathcal{H}$, action set $\mathcal{A}$,
         confidence $\delta\in(0,1)$,
         elimination aggressiveness $\alpha\in(0,1]$
\STATE Initialize $t\gets 0$
\STATE For all $i\in\mathcal{H}$, set $G_0(i)\gets\mathcal{H}\setminus\{i\}$
\STATE Set $\hat{h}(0)$ to any element of $\mathcal{H}$
\WHILE{true}
    \STATE Set \(u(t)\gets w^*(\hat h(t);G_t(\hat h(t)))\)
    \STATE Select \(A_{t+1}\) via C-Tracking toward \(u(t)\)
    \STATE Observe \(O_{t+1}\) and update \(L_{t+1}(h)\) for all \(h\in\mathcal H\)
    \STATE Set \(\hat h(t+1)\in\arg\max_{h\in\mathcal H} L_{t+1}(h)\)
    \STATE Update all relevant log-likelihood ratios $Z_{t+1}(\hat{h}(t+1),g)$
    \STATE $G_{t+1}(\hat{h}(t+1))\gets
        G_t(\hat{h}(t+1))\setminus
        \bigl\{g\in G_t(\hat{h}(t+1)):
        Z_{t+1}(\hat{h}(t+1),g)\ge
        \beta_{\mathrm{elim}}(t+1,\delta)\bigr\}$
    \STATE $G_{t+1}(i)\gets G_t(i)$ for all
               $i\neq\hat{h}(t+1)$
    \IF{$G_{t+1}(\hat{h}(t+1)) = \emptyset$}
        \RETURN $\hat{h}(t+1)$
    \ENDIF
    \STATE $t\gets t+1$
\ENDWHILE
\end{algorithmic}
\end{algorithm}

\section{Finite-Sample Analysis of Stopping Time}
\label{sec:main_theorem}
This section states finite-sample upper bounds on $\mathbb E_{h^*}[\tau]$ and provides a proof sketch for this finite-sample result. The complete proofs are given in the appendix. All probabilities and expectations are under the true hypothesis~$h^*$.
The main difficulty in analyzing $\mathbb{E}[\tau]$ under elimination is that the active-opponent process $(G_t(h^*))_{t \ge 1}$ is random: which opponents get eliminated, in what order, and when, all depend on the sample path. 

To analyze this process, we introduce the family $\mathfrak S$ of elimination chains: deterministic elimination orders for $G_t(h^*)$ under which the true hypothesis $h^*$ is never eliminated. Each chain
\[
\sigma=(S_0^\sigma \supset S_1^\sigma \supset \cdots \supset S_{K-1}^\sigma=\emptyset),
\quad
S_0^\sigma=\mathcal H\setminus\{h^*\},
\]
where each transition $S_{k-1}^\sigma$ to $S_k^\sigma$ represents the simultaneous elimination of all opponents in $S_{k-1}^\sigma \setminus S_k^\sigma$.

\begin{definition}[Chain event]
For each chain $\sigma\in\mathfrak S$, a chain event collects every sample path on which the active-opponent set follows exactly the elimination order prescribed by~$\sigma$. Define \emph{chain event}
\[
\begin{aligned}
\mathcal{E}_\sigma
:=
\Bigl\{
\hat{h}(\tau) = h^*,& \quad
\exists\, 0 = T^\sigma_0 < T^\sigma_1 < \cdots < T^\sigma_{K-1} := \tau \\
&\text{ such that }
G_t(h^*) = S_k^\sigma \\
&\ \forall\ T^\sigma_k < t \le T^\sigma_{k+1},\ k=0,\dots,K-2
\Bigr\}
\end{aligned}
\]
The chain fixes which opponents are eliminated and in what order; the transition times 
$T_1^\sigma < \cdots < T_{K-1}^\sigma$ at which these eliminations occur are random variables defined 
on~$\mathcal{E}_\sigma$ by 
$T_k^\sigma := \inf\{t > T_{k-1}^\sigma : G_t(h^*) = S_k^\sigma\}$, with $T_0^\sigma := 0$.
\end{definition}

\begin{definition}[Correct-elimination event]
Each $\mathcal{E}_\sigma$ specifies one particular chain, the union of all chain events collects every sample path on which $h^*$ survives throughout the entire elimination process. Define Correct-elimination event
\[
\mathcal{E}_{\mathrm{corr}} := \bigcup_{\sigma\in\mathfrak S} \mathcal{E}_\sigma.
\]
\end{definition}
Since each sample path determines at most one elimination pattern, the chain events are disjoint, and we have
\[
\mathbb E[\tau\,\mathbf 1_{\mathcal{E}_{\mathrm{corr}}}]
=
\sum_{\sigma\in\mathfrak S}
\mathbb E[\tau\,\mathbf 1_{\mathcal{E}_\sigma}],
\]
and hence
\[
\mathbb E[\tau]
=
\sum_{\sigma\in\mathfrak S}
\mathbb E[\tau\,\mathbf 1_{\mathcal{E}_\sigma}]
+
\underbrace{\mathbb E[\tau\,\mathbf 1_{\mathcal{E}_{\mathrm{corr}}^c}]}
_{=:\,R_{\mathrm{comp}}(\delta)}.
\]
The first term is the contribution from correct-elimination paths; $R_{\mathrm{comp}}(\delta)$ collects the stopping time contribution from paths on which $h^*$ is incorrectly eliminated at some point.

This decomposition reduces the problem of bounding $\mathbb{E}[\tau]$ to three tasks: 
(i)~for each fixed chain~$\sigma$, construct a deterministic stopping certificate via 
concentration-based \emph{good events} and convert it into an upper bound on 
$\mathbb{E}[\tau\,\mathbf{1}_{\mathcal{E}_\sigma}]$; (ii)~aggregate over all chains using their 
disjointness; (iii)~control $R_{\mathrm{comp}}(\delta)$. 

\subsection{Stopping certificate and Good-events structure}
To bound the expected stopping time, we first analyze each term $\mathbb{E}[\tau\,\mathbf{1}_{\mathcal{E}_\sigma}]$ for a fixed chain~$\sigma$; the complement $R_{\mathrm{comp}}(\delta)$ is handled in~\ref{main-result}. By the tail-sum formula, we can get
\[
\mathbb{E}[\tau\,\mathbf{1}_{\mathcal{E}_\sigma}]
= \sum_{t=0}^{\infty} 
  \Pr(\tau\,\mathbf{1}_{\mathcal{E}_\sigma} > t).
\]
On $\mathcal{E}_\sigma^c$, the indicator is zero, so 
$\tau\,\mathbf{1}_{\mathcal{E}_\sigma} = 0 \le t$; 
on $\mathcal{E}_\sigma$, the indicator is one, so 
$\tau\,\mathbf{1}_{\mathcal{E}_\sigma} = \tau$. Thus the event 
$\{\tau\,\mathbf{1}_{\mathcal{E}_\sigma} > t\}$ 
coincides with the algorithm still running at 
time~$t$ on a chain-$\sigma$ path:
\[
\{\tau\,\mathbf{1}_{\mathcal{E}_\sigma} > t\} 
= \{\tau > t\} \cap \mathcal{E}_\sigma,
\]
and hence
\[
\mathbb{E}[\tau\,\mathbf{1}_{\mathcal{E}_\sigma}]
= \sum_{t=0}^{\infty} 
  \Pr(\tau > t,\,\mathcal{E}_\sigma).
\]
To make this infinite sum tractable, we need two things: (i)~a deterministic horizon $\bar T_\sigma(\delta)$ at which to split the sum; $\bar T_\sigma(\delta)$ is defined as an explicit upper bound on the time by which the algorithm must have stopped \emph{if certain concentration conditions hold}. However, these conditions may fail on any given sample path and on such paths the algorithm may continue running past $\bar T_\sigma(\delta)$. Hence we also need
(ii)~a guarantee that beyond this horizon, each summand $\Pr(\tau > t,\,\mathcal{E}_\sigma)$ is controlled by a summable quantity. We achieve both this by constructing a sequence of \emph{Good-events} $\{\mathcal{G}_t^\sigma\}_{t \ge 1}$, which formalize these concentration conditions, and establishing the \emph{stopping implication}:
\[
\mathcal{E}_\sigma \cap \mathcal{G}_t^\sigma 
\subseteq \{\tau \le t\},
\quad \forall\, t \ge \bar T_\sigma(\delta),
\]
where $\{\tau \le t\}$ denotes the event that the algorithm has stopped by round~$t$. Taking complements gives
\[
\{\tau > t\} 
\subseteq \mathcal{E}_\sigma^c 
\cup (\mathcal{G}_t^\sigma)^c,
\]
and intersecting both sides with $\mathcal{E}_\sigma$ yields
\[
\{\tau > t\} \cap \mathcal{E}_\sigma 
\subseteq (\mathcal{G}_t^\sigma)^c.
\]
Taking probabilities, we obtain $\Pr(\tau > t,\,\mathcal{E}_\sigma) \le \Pr((\mathcal{G}_t^\sigma)^c)$ for all $t \ge \bar T_\sigma(\delta)$. Since the \emph{Good-events} are constructed so that $\sum_t \Pr((\mathcal{G}_t^\sigma)^c) < \infty$, the tail $\sum_{t \ge \bar T_\sigma(\delta)} 
\Pr(\tau > t,\,\mathcal{E}_\sigma)$ is finite, which is exactly the summability needed to close the tail-sum argument. We now describe each condition.

\textbf{(i) C-Tracking deviation.}
The sampling rule of the algorithm tracks a time-varying target allocation $u(s)\in\Delta(\mathcal A)$. Since tracking is approximate, the empirical allocation deviates from the time-average target $\bar u(t):=t^{-1}\sum_{s=1}^t u(s)$. Lemma~\ref{lem:tracking} (cf.\ Lemma~7 of~\cite{c8}) shows that C-Tracking keeps this uniformly bounded by $O(1/\sqrt{t})$. 

\textbf{(ii) Champion stabilization.}
Define the stabilization time $T_{\mathrm{ch}} := \inf\bigl\{m\ge 0: \hat h(t)=h^* \text{ for all } t\ge m\bigr\}.$ Before time $T_{\mathrm{ch}}$, the algorithm may track a wrong oracle allocation. After $T_{\mathrm{ch}}$, the allocation target matches the oracle for the current stage. Forced exploration yields a stretched-exponential tail for $T_{\mathrm{ch}}$, and the corresponding part of $\mathcal G_t^\sigma$ requires $T_{\mathrm{ch}}$ to be at most a polylogarithmic function of $t$.

\textbf{(iii) Martingale concentration.}
The log-likelihood statistic decomposes as $Z_t(h^*,g)=\sum_{s=1}^t d_{A_s}(h^*,g)+M_t^g$, where $M_t^g$ is a centered martingale. The good event requires all martingales associated with the stage sets of chain $\sigma$ to remain concentrated, with failure probability summable in~$t$. 

Putting it together, the tracking bound in (i) is deterministic; the random part of $\mathcal G_t^\sigma$ is the intersection of the stabilization and concentration events in (ii) and (iii). Their failure probabilities are summable in $t$, which is what later allows the tail-sum argument to close with only a finite additive remainder. Full definitions and proofs are given in the appendix.

\subsection{From the deterministic certificate to an expectation bound}
With the stopping implication and the summability of $\Pr((\mathcal{G}_t^\sigma)^c)$ established, we split the tail sum at $\bar T_\sigma(\delta)$:
\[
\mathbb{E}[\tau\,\mathbf{1}_{\mathcal{E}_\sigma}]
= \sum_{t=0}^{\bar T_\sigma(\delta)-1} 
  \Pr(\tau > t,\,\mathcal{E}_\sigma)
+ \sum_{t=\bar T_\sigma(\delta)}^{\infty} 
  \Pr(\tau > t,\,\mathcal{E}_\sigma).
\]
For the first term, each summand is at most $\Pr(\mathcal{E}_\sigma)$, giving $\Pr(\mathcal{E}_\sigma)\,\bar T_\sigma(\delta)$. For the second term, we showed that $\Pr(\tau > t,\,\mathcal{E}_\sigma) 
\le \Pr((\mathcal{G}_t^\sigma)^c)$, and summability yields a finite constant $C_{\mathrm{bad},\sigma}$ (Proposition~\ref{prop:good-event-summable}). Therefore
\[
\mathbb{E}[\tau\,\mathbf{1}_{\mathcal{E}_\sigma}]
\le
\Pr(\mathcal{E}_\sigma)\,\bar T_\sigma(\delta)
+ C_{\mathrm{bad},\sigma}.
\]
The problem reduces to bounding $\bar T_\sigma(\delta)$. We show that $\bar T_\sigma(\delta)$ admits an explicit upper bound in terms of the stopping threshold, the stage-wise rates, and the deviations from tracking, champion stabilization, and martingale concentration (Proposition~\ref{prop:Tbar-sigma-explicit}). Combining with the tail-sum bound above gives the following.

\begin{proposition}[Finite-sample bound along a fixed elimination chain]
\label{prop:fixed-chain}
Fix an admissible chain $\sigma\in\mathfrak S$. There exist chain-dependent constants $C_{1,\sigma},C_{2,\sigma},
C_{3,\sigma}>0$, independent of $\delta$, such that for all
$\delta\in(0,1)$,
\[
\begin{aligned}
&\mathbb E\!\left[\tau\,\mathbf 1_{\mathcal{E}_\sigma}\right]
\le
C_{3,\sigma}+ \Pr(\mathcal{E}_\sigma)\biggl(
\frac{\log(1/\delta)}{D_0}
+\frac{b}{D_0}\log\!\bigl(1+\log\tfrac{1}{\delta}\bigr)\\
&
+C_{1,\sigma}\sqrt{\log\tfrac{1}{\delta}\,\log\!\bigl(1+\log\tfrac{1}{\delta}\bigr)}
+C_{2,\sigma}\log^{2}\!\bigl(1+\log\tfrac{1}{\delta}\bigr)
\biggr).
\end{aligned}
\]
\end{proposition}

Each term corresponds to a specific component of $\bar T_\sigma(\delta)$. \textbf{Leading term} $\log\frac{1}{\delta}/D_0$: the time needed for evidence growing at rate $D_0$ to reach the dominant part $\log(1/\delta)$ of the stopping threshold. This matches the information-theoretic lower bound of~\cite{c6}. \textbf{Threshold overhead} $(b/D_0)\log(1+\log\frac{1}{\delta})$: the threshold $\beta_{\mathrm{elim}}(t,\delta) = \log\frac{1}{\delta} + b\log t + c$ grows logarithmically in~$t$ to maintain time-uniform validity; the algorithm stops when evidence exceeds this threshold. At the stopping horizon $t \approx \log\frac{1}{\delta}/D_0$, the $b\log t$ term contributes $(b/D_0)\log(1+\log\frac{1}{\delta})$. \textbf{Tracking and concentration cost} $C_{1,\sigma}\sqrt{\log\frac{1}{\delta}\log(1+\log\frac{1}{\delta})}$: the combined effect of C-Tracking deviation~(i) and martingale fluctuation~(iii). The constant $C_{1,\sigma}$ depends on the Lipschitz constant of the information functional and the martingale union-bound factor, both evaluated over the final active set $S_K^\sigma$. Since both quantities are nonincreasing as the opponent set shrinks, $C_{1,\sigma}$ can be smaller for chains with more elimination. \textbf{Stabilization cost} $C_{2,\sigma}\log^{2}(1+\log\frac{1}{\delta})$: the effect of champion stabilization~(ii). The stretched-exponential tail of $T_{\mathrm{ch}}$ requires $O(\log^2 t)$ rounds for the good-event failure probability to be summable. \textbf{Constant} $C_{3,\sigma}$: absorbs $C_{\mathrm{bad},\sigma}$ from the tail-sum splitting and other finite, $\delta$-independent terms.

\subsection{Aggregation and main result}
\label{main-result}
The Proposition~\ref{prop:fixed-chain} bounds each chain separately; now we sum over all chains to obtain a bound on the correct elimination event.

\begin{proposition}[From fixed chains to full expectation]
\label{prop:aggregation}
Summing Proposition~\ref{prop:fixed-chain} over the disjoint family $\{\mathcal{E}_\sigma\}_{\sigma\in\mathfrak S}$ and using $\sum_\sigma \Pr(\mathcal{E}_\sigma) = \Pr(\mathcal{E}_{\mathrm{corr}}) \le 1$ to factor out the common leading and $\log(1+\log\frac{1}{\delta})$ terms gives
\[
\begin{aligned}
\mathbb E[\tau\,\mathbf 1_{\mathcal{E}_{\mathrm{corr}}}]
&\le 
\Pr(\mathcal{E}_{\mathrm{corr}})
\left(
\frac{\log(1/\delta)}{D_0}+\frac{b}{D_0}\log\!\bigl(1+\log\tfrac{1}{\delta}\bigr)
\right)
\\
&\quad+\sum_{\sigma\in\mathfrak S}\Pr(\mathcal{E}_\sigma)\, C_{1,\sigma}\sqrt{\log\tfrac{1}{\delta}\,\log\!\bigl(1+\log\tfrac{1}{\delta}\bigr)}\\
&\quad+
\sum_{\sigma\in\mathfrak S}\Pr(\mathcal{E}_\sigma)\,C_{2,\sigma}\log^{2}\!\bigl(1+\log\tfrac{1}{\delta}\bigr) + \sum_{\sigma\in\mathfrak S} C_{3,\sigma},
\end{aligned}
\]
where $\sum_{\sigma\in\mathfrak S} C_{3,\sigma}<\infty$ since $\mathfrak{S}$ is finite and each $C_{3,\sigma}$ is a finite constant independent of~$\delta$.
\end{proposition}

Recall the decomposition of $\mathbb{E}[\tau]$; it remains to bound $R_{\mathrm{comp}}(\delta) = \mathbb{E}[\tau\,\mathbf{1}_{\mathcal{E}_{\mathrm{corr}}^c}] = \sum_{t=0}^{\infty} 
\Pr(\tau > t,\,\mathcal{E}_{\mathrm{corr}}^c)$. The \emph{Good-events} $\widetilde{\mathcal{G}}_t$ includes 
champion stabilization $\hat{h}(t)=h^*$, which forces correct output; hence on $\widetilde{\mathcal{G}}_t$ the algorithm has stopped with output~$h^*$, and in particular $\widetilde{\mathcal{G}}_t \subseteq \{\tau \le t\}$. This holds for all $t \ge \widetilde T(\delta)$, the smallest horizon at which the evidence lower bound (at the base rate~$D_0$) exceeds $\beta_{\mathrm{elim}}$. Since $\widetilde T(\delta) = O(\log(1/\delta)/D_0)$, 
splitting at $\widetilde T(\delta)$: before, each term is at most $\Pr(\mathcal{E}_{\mathrm{corr}}^c)$; 
after, $\Pr(\tau>t,\,\mathcal{E}_{\mathrm{corr}}^c) \le \Pr(\tau > t) \le \Pr(\widetilde{\mathcal{G}}_t^c)$, 
which is summable. Define $C_{\widetilde{\mathcal{G}}} := 
\sum_{t \ge 3}
\Pr(\widetilde{\mathcal{G}}_t^c) < \infty$. Therefore
\[
R_{\mathrm{comp}}(\delta)
\le
\widetilde T(\delta)\cdot
\Pr(\mathcal{E}_{\mathrm{corr}}^c)
+ C_{\widetilde{\mathcal{G}}}.
\]
Under $\alpha=1$, 
$\Pr(\mathcal{E}_{\mathrm{corr}}^c)\le \delta$ 
(Theorem~\ref{thm:delta-pac-correctness}), so 
$R_{\mathrm{comp}}(\delta) 
= O(\delta \log(1/\delta)) + O(1) = O(1)$, 
absorbed into~$C_3$.
The detailed illustration is in Proposition~\ref{prop:Rcomp-small}.

\begin{theorem}[Finite-sample upper bound in the $\delta$-PAC regime]
\label{thm:main-delta-pac}
Assume $\alpha =1$, and the conditions of Section~\ref{subsec:adaptive_elimination}, the expected stopping time satisfies
\[
\begin{aligned}
\mathbb E[\tau]
&\le
\frac{\log\frac{1}{\delta}}{D_0}
+\frac{b}{D_0}\log\!\bigl(1+\log\tfrac{1}{\delta}\bigr) + \widetilde C_3 \\
&+ \sum_\sigma\Pr(\mathcal{E}_\sigma)
\widetilde C_{1,\sigma}\sqrt{\log\tfrac{1}{\delta}\,
\log\!\bigl(1+\log\tfrac{1}{\delta}\bigr)}\\
&+\sum_\sigma\Pr(\mathcal{E}_\sigma)\widetilde C_{2,\sigma}\log^{2}\!\bigl(1+\log\tfrac{1}{\delta}\bigr),
\end{aligned}
\]
\end{theorem}

\begin{remark}[Comparison with Track-and-Stop]
\label{rem:comparison-tas}
Applying the same proof technique to the Track-and-Stop algorithm yields a bound of the same form. The leading term and sub-leading term are identical. Comparing the remaining terms:
\begin{itemize}
\item \emph{$\sqrt{\log\frac{1}{\delta}\,\log(1+\log\frac{1}{\delta})}$ coefficient $C_{1,\sigma}$}:
  \textbf{strictly improved.} This constant depends on the Lipschitz
  constant of the information functional and the martingale envelope,
  both evaluated over the final active set $S_{K-2}^\sigma$.
  Since $S_{K-2}^\sigma\subset S_0$, both quantities are
  nonincreasing, giving $C_{1,\sigma}$ becomes smaller whenever
  elimination occurs.

\item \emph{$\log^{2}(1+\log\frac{1}{\delta})$ coefficient 
  $C_{2,\sigma}$}: comparable to TaS. This term is 
  governed by champion stabilization, which completes 
  in $O(\log^2 t)$ rounds---well before the first 
  elimination at $\Theta(\log(1/\delta))$. Since 
  stabilization precedes elimination, its cost is 
  insensitive to whether elimination is used, and 
  this coefficient is not expected to improve.
\end{itemize}

The $\sqrt{\log\frac{1}{\delta}\,\log(1+\log\frac{1}{\delta})}$ scale is therefore the highest order at which elimination provides a provable finite-sample improvement over TaS.
\end{remark}

Next, we state the stopping-time analog in the aggressive-elimination regime $\alpha\in(0,1)$.
\begin{corollary}[Sample complexity under 
aggressive elimination]
\label{cor:alpha-less-than-one}
Assume $\alpha\in(0,1)$ and the conditions of 
Section~\ref{subsec:adaptive_elimination}. Then for 
every $\delta\in(0,1)$, the expected stopping time 
satisfies
\[
\begin{aligned}
\mathbb E[\tau]
&\le
\frac{\alpha \log\frac{1}{\delta}}{D_0}
+\frac{b}{D_0}\log\!\bigl(1+\log\tfrac{1}{\delta}\bigr) + \widetilde C_3 \\
&+ \sum_\sigma\Pr(\mathcal{E}_\sigma)
\widetilde C_{1,\sigma}\sqrt{\log\tfrac{1}{\delta}\,
\log\!\bigl(1+\log\tfrac{1}{\delta}\bigr)}\\
&+\sum_\sigma\Pr(\mathcal{E}_\sigma)\widetilde C_{2,\sigma}\log^{2}\!\bigl(1+\log\tfrac{1}{\delta}\bigr),
\end{aligned}
\]
and the correctness guarantee relaxes to
\[
\Pr(\hat h(\tau)\neq h^*)\le c_{\mathrm{elim}}
\delta^\alpha,
\]
where $\widetilde C_{1,\sigma}, \widetilde C_{2,\sigma}, 
\widetilde C_3, c_{\mathrm{elim}}>0$ are chain-dependent 
constants independent of~$\delta$.
\end{corollary}

\begin{remark}[Speed--safety tradeoff]
\label{rem:alpha-tradeoff}
This bound concerns stopping time only. At $\alpha=1$, the algorithm is $\delta$-PAC with leading term $\log(1/\delta)/D_0$. Reducing $\alpha$ has two effects: it lowers the leading term to $\alpha\log(1/\delta)/D_0$ by decreasing the elimination threshold, but weakens the error 
guarantee from $\delta$ to $O(\delta^\alpha)$.
\end{remark}

\begin{figure}[t]
\centering
\begin{tikzpicture}[
    >=stealth,
    x=0.62cm, y=0.72cm,
    font=\scriptsize
]
\def\xmax{9.8}
\def\Tone{1.8}  \def\Ttwo{3.2}  \def\Tthree{4.3}
\def\TK{5.8}    \def\Tstop{8.8}
\def\Dzero{1.0} \def\Done{1.55} \def\Dtwo{2.0} \def\DK{2.8}
\def\yStage{4.2}


\fill[blue!7]  (0,\yStage-0.38) rectangle (\Tone,\yStage+0.38);
\fill[blue!14] (\Tone,\yStage-0.38) rectangle (\Ttwo,\yStage+0.38);
\fill[blue!21] (\Ttwo,\yStage-0.38) rectangle (\Tthree,\yStage+0.38);
\fill[gray!12] (\Tthree,\yStage-0.38) rectangle (\TK,\yStage+0.38);
\fill[blue!32] (\TK,\yStage-0.38) rectangle (\Tstop,\yStage+0.38);

\foreach \xa/\xb in {
    0/\Tone, \Tone/\Ttwo, \Ttwo/\Tthree,
    \Tthree/\TK, \TK/\Tstop}{
  \draw[gray!40, thin] (\xa,\yStage-0.38) rectangle (\xb,\yStage+0.38);
}

\node [font=\tiny, blue!70!black, inner sep=1pt]
    at ({0.5*\Tone}, \yStage) {\shortstack{$G_t(h^*)$\\$=S_0$}};
\node[font=\tiny, blue!70!black, inner sep=1pt]
    at ({0.5*\Tone+0.5*\Ttwo}, \yStage) {$S_1$};
\node[font=\tiny, blue!70!black, inner sep=1pt]
    at ({0.5*\Ttwo+0.5*\Tthree}, \yStage) {$S_2$};
\node[font=\tiny, gray!55, inner sep=1pt]
    at ({0.5*\Tthree+0.5*\TK}, \yStage) {$\cdots$};
\node[font=\tiny, blue!70!black, align=center, inner sep=1pt]
    at ({0.5*\TK+0.5*\Tstop}, \yStage) {\shortstack{$G_t(h^*)=S_{K-1}$}};

\node[left, font=\tiny, align=right] at (0, \yStage)
    {\textbf{active}\\\textbf{set}};

\foreach \x/\gl in {
    \Tone/{$g_{1}$}, \Ttwo/{$g_{2}$},
    \Tthree/{$g_{3}$}, \TK/{$g_{K}$}}{
  \draw[gray!45, dashed, thin] (\x,{\yStage-0.35}) -- (\x,{\yStage+0.35});
  \node[font=\tiny, red!55] at (\x,\yStage + 0.6)
      {\textbf{$\times$}};
  \node[font=\tiny, red!65] at (\x,\yStage+0.85) {\gl};
}

\draw[green!50!black, thin] (\Tstop,0) -- (\Tstop,{\yStage-0.85});
\draw[green!50!black, thin] (\Tstop,{\yStage-0.35}) -- (\Tstop,{\yStage+0.35});
\node[font=\tiny, green!50!black] at (\Tstop,\yStage+0.50)
    {\textbf{stop}};

\draw[gray!20] (0,\yStage-0.95) -- (\xmax,\yStage-0.95);


\draw[->] (0,0) -- (\xmax+0.4,0) node[right] {$t$};
\draw[->] (0,0) -- (0,3.3);

\node[rotate=90, font=\tiny, black] at (-0.3,1.8)
    {Information Rate $D_k$};

\foreach \d/\lab in {
    \Dzero/{$D_0$}, \Done/{$D_1$},
    \Dtwo/{$D_2$}, \DK/{$D_K$}}{
  \draw[gray!45, dashed, thin] (0,\d) -- ({\Tstop-0.15},\d);
  \node[anchor=west, font=\tiny, text=black,
        fill=white, fill opacity=0.85, text opacity=1, inner sep=1pt]
        at ({\Tstop+0.08},\d) {\lab};
}

\draw[blue!55, thick] (0,\Dzero) -- (\Tstop,\Dzero);
\node[font=\scriptsize\sffamily, text=blue!55,
      fill=white, fill opacity=0.85, text opacity=1, inner sep=1pt,
      anchor=south]
  at ({0.5*\Tstop},\Dzero) {baseline};

\fill[red!12]
    (0,0.20) -- (0.37,0.577) -- (0.73,0.85) -- (1.1,0.89)
    -- (1.47,0.907) -- (\Tone,0.94)
    -- (\Tone,\Dzero) -- (0,\Dzero) -- cycle;

\fill[green!12]
    (\Tone,\Dzero)
    -- (1.83,0.94) -- (2.2,1.123) -- (2.57,1.196) -- (2.93,1.225)
    -- (\Ttwo,1.315)
    -- (3.67,1.477) -- (4.03,1.537)
    -- (\Tthree,1.577)
    -- (4.77,1.638) -- (5.13,1.675) -- (5.5,1.698)
    -- (\TK,1.762)
    -- (6.23,2.007) -- (6.6,2.194) -- (6.97,2.337)
    -- (7.33,2.446) -- (7.7,2.529) -- (8.07,2.593)
    -- (8.43,2.642) -- (\Tstop,2.679)
    -- (\Tstop,\Dzero) -- cycle;

\draw[orange!85!red, very thick, smooth] plot coordinates {
    (0,0.20)(0.37,0.577)(0.73,0.85)(1.1,0.89)
    (1.47,0.907)(\Tone,0.94)
    (1.83,0.94)(2.2,1.123)(2.57,1.196)(2.93,1.225)
    (\Ttwo,1.315)
    (3.67,1.477)(4.03,1.537)
    (\Tthree,1.577)
    (4.77,1.638)(5.13,1.675)(5.5,1.698)
    (\TK,1.762)
    (6.23,2.007)(6.6,2.194)(6.97,2.337)
    (7.33,2.446)(7.7,2.529)(8.07,2.593)
    (8.43,2.642)(\Tstop,2.679)
};

\node[font=\tiny, red!55, align=center] at (1.0,0.55)
    {warm-up};

\draw[->, orange!70!red, thin] (7.7,3.0) -- (8.3,2.65);
\node[font=\tiny, orange!70!red] at (7.1,3.15)
    {approaches $D_K$};

\foreach \x in {\Tone,\Ttwo,\Tthree,\TK}{
  \draw[gray!45, dashed, thin] (\x,0) -- (\x,{\yStage-0.85});
}

\foreach \x/\lab in {
    \Tone/{$T_1$}, \Ttwo/{$T_2$}, \Tthree/{$T_3$},
    \TK/{$T_K$}, \Tstop/{$T_{\sigma}$}}{
  \node[below, font=\tiny, black] at (\x,0) {\lab};
}
\end{tikzpicture}

\caption{Elimination Mechanism.
\textbf{Upper}: the active-opponent set $G_t(h^*)$ shrinks as opponents
are eliminated ($\times$);
\textbf{Lower}: the information rate (orange) starts below $D_0$ during
warm-up (red), then rises after each elimination, approaching $D_K$.}

\label{fig:proof-roadmap}
\end{figure}
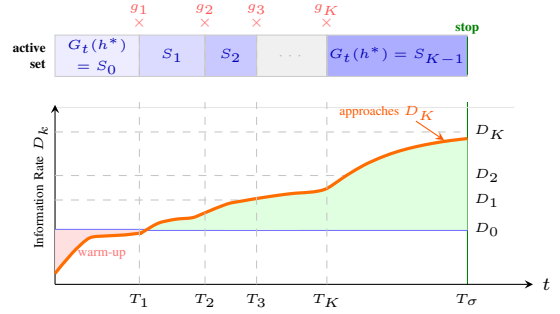

\section{EXPERIMENTAL VALIDATION}
\label{sec:experiments}
We evaluate the proposed elimination-based Track-and-Stop algorithm on
synthetic Gaussian active-hypothesis-testing instances.
The experiments are designed to answer three questions:
(i) whether elimination reduces the stopping time in the exact-confidence regime ($\alpha=1$),
(ii) how the elimination parameter $\alpha$ controls the time–error trade-off, and
(iii) whether the reduction in expected stopping time arises solely from earlier stopping or also from elimination-induced changes in the sampling allocation.

\subsection{Experimental Setup}

We consider a finite active hypothesis testing problem with
$|\mathcal H|=5$ hypotheses and $|\mathcal A|=5$ sensing actions.
For each action $a\in\mathcal A$ and hypothesis $h\in\mathcal H$, the
observation distribution is Gaussian: $O_t\sim \mathcal N(\mu_{a,h},\sigma^2).$

Unless otherwise stated, each reported stopping-time statistic is averaged
over $1000$ independent Monte Carlo trials.
We test three environments designed to exercise qualitatively different
elimination regimes:
a \emph{skewed} environment in which several hypotheses become easy to
eliminate early;
a \emph{hard-weak} environment in which the bottleneck competitors remain
active for most of the run;
and a \emph{degenerate} environment in which some actions are nearly
uninformative.

We compare four policies:
\begin{itemize}
    \item \textbf{Greedy}, which basically selects the action with the largest instantaneous divergence between the current champion and its most competitive rival~\cite{c12}.
    \item \textbf{Track-and-Stop (TaS)}, which tracks the oracle
    allocation computed against the full opponent set.
    \item \textbf{TaS + stopping-only elimination}, which is called \textit{StopElim} and keeps the TaS sampling rule but replaces the stopping rule with an elimination-based stopping.
    \item \textbf{TaS + elimination-aware sampling (ours)}, which will be called \textit{FullElim}, which uses
    both elimination-based stopping and active-set-aware oracle tracking.
\end{itemize}
The stopping-only elimination baseline is important for question (iii) because it separates the gain from earlier stopping from the additional gain due to recomputing the sampling allocation on the reduced active set.

\subsection*{Experiment I: $\delta$-Confidence Stopping Time Comparison}
\addcontentsline{toc}{subsection}{Experiment I: $\delta$-Confidence stopping-time comparison}

Here, we aim to determine whether elimination yields a finite-sample
stopping-time reduction while remaining consistent with the prescribed confidence target.
We evaluate this $\delta$-confidence regime by setting $\alpha=1$ and
sweeping the confidence level $\delta \in \{0.1, 0.05, 0.01, 0.005, 0.001\}$ across the three environments.  The results on average stopping time and error rate across these environments in Appendix A.
\begin{figure}[t]
    \centering
    \includegraphics[width=1\linewidth]{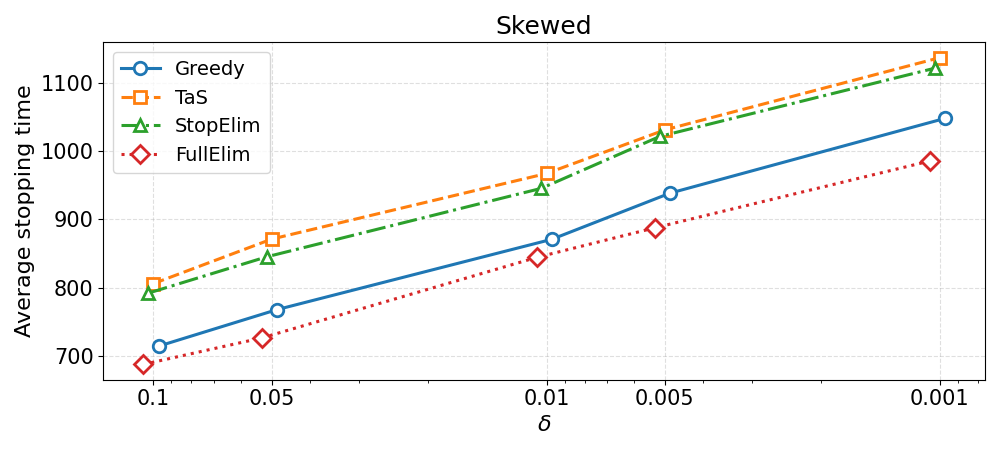}\\\vspace{-0.3em}
    \includegraphics[width=1\linewidth]{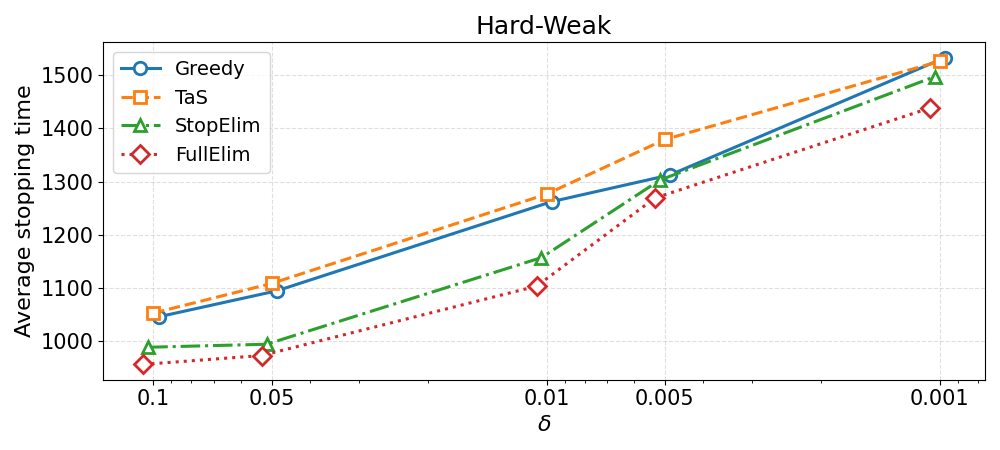}\\\vspace{-0.3em}
    \includegraphics[width=1\linewidth]{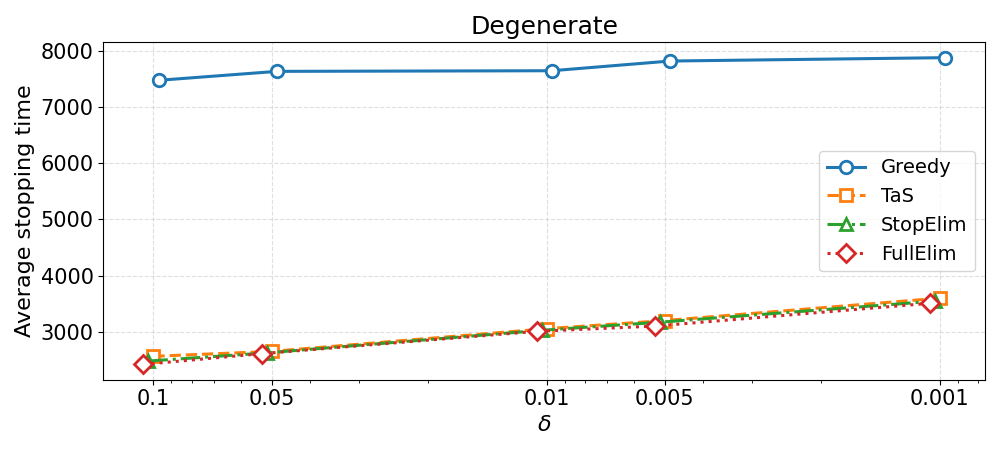}\\\vspace{-0.8em}
    \caption{Exact-confidence stopping-time comparison versus $\delta$ in the skewed, hard-weak, and degenerate environments.}
    \label{fig:exact_stopping}
\end{figure}

Figure 2 shows that FullElim consistently outperforms TaS. In the skewed environment, stopping time at $\delta = 0.05$ drops by $19.9\%$ (871 to 726), reflecting the acceleration from early elimination and the corresponding increase in effective information rate from $D_0$ to $D_K$. In the hard-weak environment, the gain is smaller ($13.8\%$, 1107 to 972), and StopElim closely matches FullElim, as persistent bottleneck rivals limit active-set contraction and reduce reallocation benefits. In the degenerate environment, Greedy requires over 7900 samples versus roughly 3500 for elimination-based methods and exhibits a high average error of 0.6, while elimination methods maintain the $\delta$-PAC guarantee. Greedy oversamples uninformative actions, whereas TaS-based methods focus effort on informative actions via the oracle allocation $w^*$.
\subsection*{Experiment II: Relaxed-Confidence Time-Error Trade-off}
\addcontentsline{toc}{subsection}{Experiment II: Relaxed-confidence time--error trade-off}
Next, we study the relaxed-confidence regime. This experiment is intended to test the predicted trade-off between speed and reliability under more aggressive elimination. Fixing $\delta=0.1$ in all three settings, we vary the elimination parameter $\alpha \in \{0.2,\,0.4,\,0.6,\, 0.8, \, 1\}$ and record both the stopping time and the error rate. Detailed result of all three settings is included in Appendix C (table III, IV, V).
\begin{figure}[t]
    \centering
    \includegraphics[width=\linewidth]{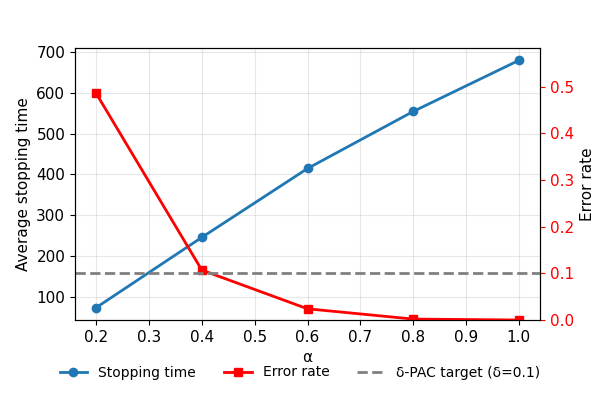}\\\vspace{-0.5em}
    \caption{\looseness=-1
    Relaxed-confidence trade-off in the skewed setting: stopping time grows and error rate decreases with $\alpha$.
    }
    \label{fig:tradeoff_alpha}
\end{figure}
Figure~\ref{fig:tradeoff_alpha} shows a clear monotone trade-off in the skewed setting (similar results hold for hard-weak and degenerate settings, as shown in appendix). Stopping time decreases from $679$ at $\alpha=1$ to $74$ at $\alpha=0.2$ (an $89\%$ reduction), while error increases from $\approx 0$ to $0.486$.
The error constraint $\Pr(\hat h\neq h^*)\le\delta=0.1$ is satisfied for $\alpha\ge 0.6$ and violated for $\alpha\le 0.4$. This transition is consistent with the theoretical bound: the elimination calibration guarantees $\Pr(C_{\mathrm{elim}}^c)\le c'\delta^\alpha$, so smaller $\alpha$ relaxes the elimination confidence requirement and increases the risk of a wrong elimination. The empirical threshold near $\alpha\approx 0.5$ identifies the operating point below which the elimination error dominates the stopping gain.

\subsection*{Experiment III: Elimination Mechanism Validation}
\addcontentsline{toc}{subsection}{Experiment III: Mechanism validation via active-set and allocation dynamics}
To illustrate how FullElim reduces stopping time, we record internal dynamics for a representative trial: active-set membership, empirical sampling allocation, evidence accumulation versus the elimination
threshold $\beta_{\mathrm{elim}}(t)$, and instantaneous information rate.

\begin{figure}[t]
    \centering
    \includegraphics[width=1\linewidth]{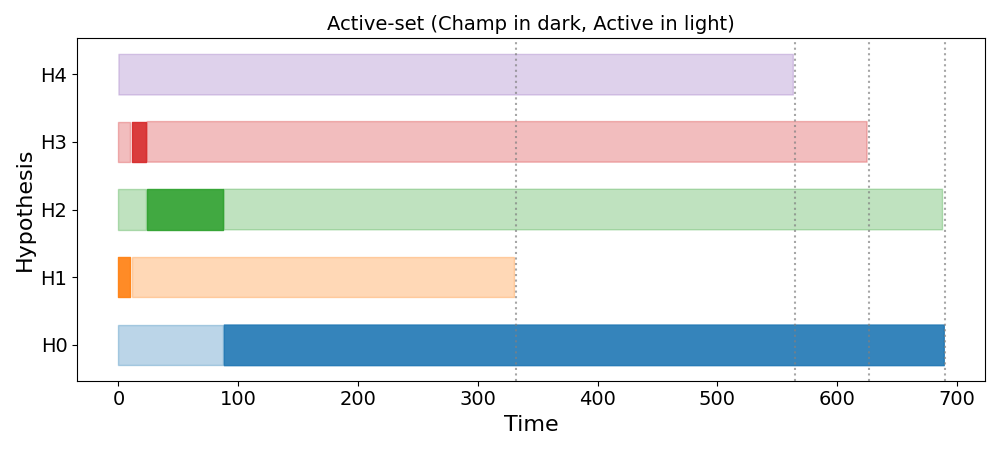}\\\vspace{-0.3em}   \includegraphics[width=1\linewidth]{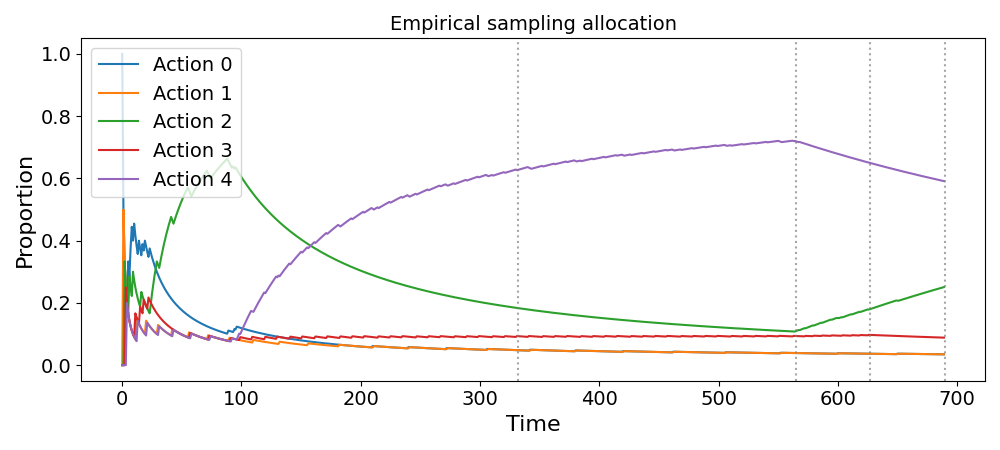}\\\vspace{-0.3em}
    \includegraphics[width=1\linewidth]{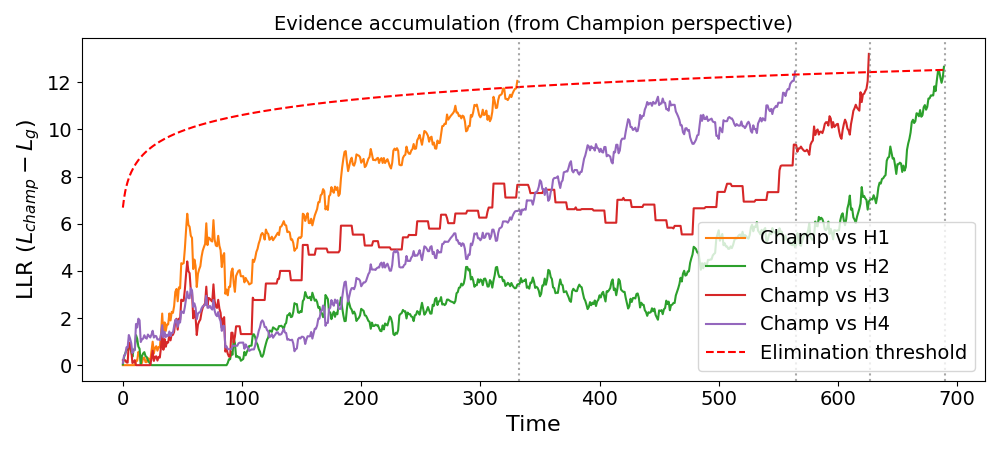}\\\vspace{-0.3em}\includegraphics[width=1\linewidth]{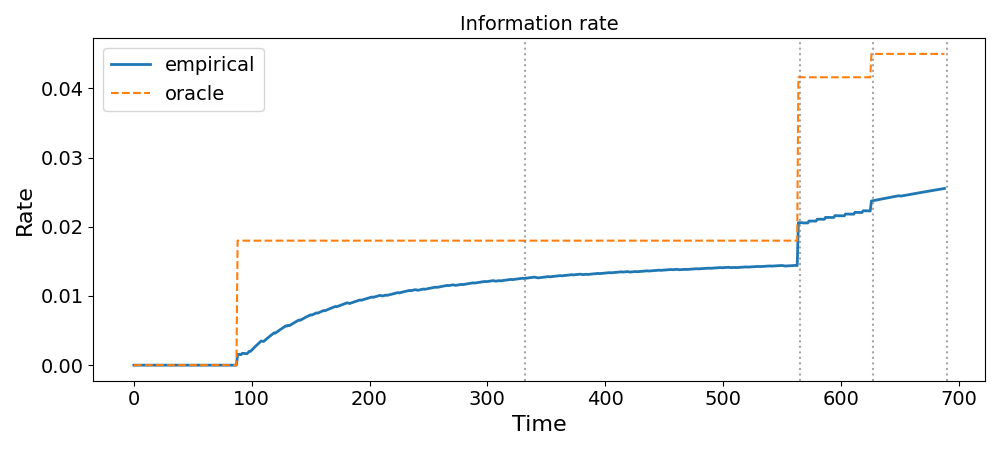}\\\vspace{-0.9em}
    \caption{Mechanism diagnostics with grey global elimination line.
    First: active-set dynamics over time.
    Second: empirical sampling allocation.
    Third: evidence accumulation of chamion against opponents in active set.
    Fourth: oracle vs empirical information rate.
    }
    \label{fig:mechanism}
\end{figure}
Figure~\ref{fig:mechanism} visualizes the shrinkage of the active set and the corresponding change in empirical sampling allocation. The oracle rate (fourth plot) exhibits discrete staircase jumps at $t\approx$ $560$, and $620$, each coinciding with an elimination event in the active-set panel. Each jump reflects the theoretical stage-wise increase $D_0\le D_1\le\cdots\le D_K$: as opponents are eliminated, the oracle solves a lower-dimensional max-min problem whose value is strictly higher. The empirical rate tracks these jumps, confirming that C-Tracking successfully re-targets the new optimal allocation in finite time.
 
The reallocation mechanism is visible in the sampling-proportion panel (second plot): after $H_1$ and $H_4$ are eliminated, the algorithm shifts effort from action $4$ to concentrates on action $2$ for resolving the remaining bottleneck. This mechanism underlies the stopping-time reductions observed in Experiment~I.

\section{CONCLUSION}
This paper proposed an elimination-augmented Track-and-Stop algorithm for active multi-hypothesis testing that progressively prunes the implausible hypotheses and reallocates sensing effort toward the surviving alternatives.
We proved a finite-sample upper bound on the expected stopping time, showing that elimination provably
reduces the $\sqrt{\log(1/\delta)\cdot\log\log(1/\delta)}$ coefficient. We also showed that an aggressiveness parameter $\alpha$ trades off a faster leading term $\alpha\log(1/\delta)/D_0$ against a relaxed error guarantee of order $O(\delta^\alpha)$. Our synthetic experiments demonstrate that elimination consistently reduces
stopping time, that the $\alpha$-parameter induces a monotone speed-reliability trade-off, and that the gain arises from both earlier stopping and reallocation of sensing effort to more informative actions. 

Several limitations should be noted. The lower-order constants in the bound were not optimized, and the cancellation at the $\log\log(1/\delta)$ scale may be an artifact of the matched-threshold design rather than a fundamental barrier. Also, our empirical study is limited to synthetic Gaussian instances. Interesting directions for future work include investigating whether separated elimination and stopping thresholds can break that cancellation and designing stage-dependent elimination schedules. Inspired by the system cost in anomaly detection~\cite{c13}, we can extend our framework to cost-aware action selection where sensing cost, elimination aggressiveness, and stopping time interact in a three-way trade-off.

\section*{APPENDIX}
\subsection{Experimental Design and Results}

We consider Gaussian observation models with unit variance. For each action $a \in \mathcal{A}$ and hypothesis $h \in \mathcal{H}$,
\[
O_t \mid (A_t = a, h^* = h) \sim \mathcal{N}(\mu_{a,h}, 1),
\]
where the mean matrix $\mu \in \mathbb{R}^{|\mathcal{A}| \times |\mathcal{H}|}$ fully specifies the environment.

We study three environments designed to capture different statistical structures of the identification problem: (i) a skewed environment with hypothesis-dependent informative actions, (ii) a hard-weak environment with small pairwise separations among a subset of hypotheses, and (iii) a degenerate environment with partially uninformative actions. The true hypothesis is $H_0$, for all three settings.

\paragraph{Skewed Environment}
\[
\mu =
\begin{bmatrix}
0.5 & 0.9 & 0.5 & 0.3 & 0.7 \\
0.3 & 0.5 & 0.3 & 0.5 & 0.3 \\
0.5 & 0.2 & 0.5 & 0.3 & 0.8 \\
0.7 & 0.3 & 0.7 & 0.1 & 0.5 \\
0.4 & 0.6 & 0.6 & 0.4 & 0.2
\end{bmatrix}
\]

\paragraph{Hard-Weak Environment}
\[
\mu =
\begin{bmatrix}
0.9 & 0.8 & 0.2 & 0.2 & 0.2 \\
0.8 & 0.65 & 0.2 & 0.2 & 0.2 \\
0.1 & 0.1 & 0.8 & 0.1 & 0.1 \\
0.2 & 0.2 & 0.1 & 0.8 & 0.2 \\
0.1 & 0.2 & 0.1 & 0.2 & 0.9
\end{bmatrix}
\]

\paragraph{Degenerate Environment}
\[
\mu =
\begin{bmatrix}
0.5 & 0.9 & 0.1 & 0.5 & 0.5 \\
0.5 & 0.1 & 0.9 & 0.5 & 0.5 \\
0.5 & 0.5 & 0.5 & 0.5 & 0.5 \\
0.5 & 0.5 & 0.5 & 0.5 & 0.5 \\
0.55 & 0.45 & 0.45 & 0.45 & 0.45
\end{bmatrix}
\]

\subsubsection{Results of Experiment I: $\delta$-Confidence stopping-time comparison}

We report the raw results across the three environments for varying confidence parameter $\delta$. Two tables are provided: average stopping time in table \ref{tab:stopping_time_exp1} and error rate in table \ref{tab:error_rate_exp1}.
\begin{table}[h!]
\centering
\caption{Average stopping time for different environments and $\delta$ (Experiment I).}
\label{tab:stopping_time_exp1}
\begin{tabular}{c|c|c|c|c|c}
\hline
Environment & $\delta$ & Greedy & TaS & StopElim & FullElim \\
\hline
\multirow{5}{*}{Skewed} 
 & 0.1   & 714.68 & 806.15 & 792.22 & 688.06 \\
 & 0.05  & 768.09 & 871.27 & 845.32 & 726.73 \\
 & 0.01  & 870.94 & 967.21 & 945.54 & 844.42 \\
 & 0.005 & 938.49 & 1030.98 & 1021.42 & 887.80 \\
 & 0.001 & 1047.74 & 1136.57 & 1121.13 & 985.43 \\
\hline
\multirow{5}{*}{Hard-Weak} 
 & 0.1   & 1044.90 & 1051.96 & 987.59 & 955.29 \\
 & 0.05  & 1093.98 & 1107.88 & 993.30 & 972.00 \\
 & 0.01  & 1262.06 & 1276.15 & 1156.55 & 1102.74 \\
 & 0.005 & 1312.46 & 1379.27 & 1302.67 & 1268.89 \\
 & 0.001 & 1533.43 & 1526.59 & 1497.60 & 1438.44 \\
\hline
\multirow{5}{*}{Degenerate} 
 & 0.1   & 7476.15 & 2521.90 & 2477.62 & 2418.23 \\
 & 0.05  & 7635.45 & 2649.66 & 2623.12 & 2609.48 \\
 & 0.01  & 7646.71 & 3053.86 & 3025.22 & 3006.29 \\
 & 0.005 & 7819.37 & 3178.59 & 3169.22 & 3150.54 \\
 & 0.001 & 7879.25 & 3597.51 & 3543.84 & 3535.65 \\
\hline
\end{tabular}
\end{table}
\begin{table}[h!]
\centering
\caption{Error rates for different environments and $\delta$ (Experiment I).}
\label{tab:error_rate_exp1}
\begin{tabular}{c|c|c|c|c|c}
\hline
Environment & $\delta$ & Greedy & TaS & StopElim & FullElim \\
\hline
\multirow{5}{*}{Skewed} 
 & 0.1   & 0.0030 & 0.0020 & 0.0010 & 0.0010 \\
 & 0.05  & 0.0020 & 0.0010 & 0.0010 & 0.0000 \\
 & 0.01  & 0.0020 & 0.0010 & 0.0000 & 0.0010 \\
 & 0.005 & 0.0000 & 0.0000 & 0.0000 & 0.0000 \\
 & 0.001 & 0.0000 & 0.0000 & 0.0000 & 0.0000 \\
\hline
\multirow{5}{*}{Hard-Weak} 
 & 0.1   & 0.0000 & 0.0000 & 0.0000 & 0.0000 \\
 & 0.05  & 0.0000 & 0.0000 & 0.0000 & 0.0000 \\
 & 0.01  & 0.0000 & 0.0000 & 0.0000 & 0.0000 \\
 & 0.005 & 0.0000 & 0.0000 & 0.0000 & 0.0000 \\
 & 0.001 & 0.0000 & 0.0000 & 0.0000 & 0.0000 \\
\hline
\multirow{5}{*}{Degenerate} 
 & 0.1   & 0.6490 & 0.0000 & 0.0020 & 0.0010 \\
 & 0.05  & 0.6380 & 0.0000 & 0.0000 & 0.0000 \\
 & 0.01  & 0.6170 & 0.0000 & 0.0020 & 0.0000 \\
 & 0.005 & 0.6070 & 0.0010 & 0.0000 & 0.0000 \\
 & 0.001 & 0.6590 & 0.0000 & 0.0010 & 0.0000 \\
\hline
\end{tabular}
\end{table}

\subsubsection{Results of Experiment II: Relaxed-Confidence Time-Error Trade-off}

In this experiment, we fix $\delta = 0.1$ and vary the elimination parameter $\alpha \in \{0.2,0.4,0.6,0.8,1.0\}$ under the all three environments. We record both the stopping time and the empirical terminal error rate and put the results into table \ref{tab:stopping_time_exp2_skewed}, \ref{tab:stopping_time_exp2_hw}, and \ref{tab:stopping_time_exp2_deg}.

\begin{table}[h!]
\centering
\caption{Stopping time for different $\alpha$ (Experiment II, Skewed environment, $\delta=0.1$).}
\label{tab:stopping_time_exp2_skewed}
\begin{tabular}{c|c|c}
\hline
$\alpha$ & Stopping time & Error rate \\
\hline
0.20 & 74.37 & 0.4860\\
0.40 & 246.13 & 0.1070\\
0.60 & 414.50 & 0.0240\\
0.80 & 553.91 & 0.0020\\
1.00 & 678.83 & 0.0000\\
\hline
\end{tabular}
\end{table}
\begin{table}[h!]
\centering
\caption{Stopping time for different $\alpha$ (Experiment II, Hard-Weak environment, $\delta=0.1$).}
\label{tab:stopping_time_exp2_hw}
\begin{tabular}{c|c|c}
\hline
$\alpha$ & Stopping time & Error rate \\
\hline
0.20 & 100.44 & 0.3580\\
0.40 & 335.88 & 0.0770\\
0.60 & 575.91 & 0.0160\\
0.80 & 824.90 & 0.0050\\
1.00 & 962.10 & 0.0000\\
\hline
\end{tabular}
\end{table}
\begin{table}[h!]
\centering
\caption{Stopping time for different $\alpha$ (Experiment II, Degenerate environment, $\delta=0.1$).}
\label{tab:stopping_time_exp2_deg}
\begin{tabular}{c|c|c}
\hline
$\alpha$ & Stopping time & Error rate \\
\hline
0.20 & 452.07 & 0.3300\\
0.40 & 981.36 & 0.0610\\
0.60 & 1425.22 & 0.0100\\
0.80 & 1959.07 & 0.0010\\
1.00 & 2427.38 & 0.0000\\
\hline
\end{tabular}
\end{table}
\subsection{Finite Sample Analysis}

This appendix provides complete proofs for all finite-sample results stated in 
Section~\ref{sec:main_theorem} of the main text. The structure is centered around establishing 
Theorem~\ref{thm:main-delta-pac} (the expectation bound under the $\delta$-PAC regime $\alpha=1$); along 
the way, the supporting propositions (Propositions~\ref{prop:fixed-chain} and 
\ref{prop:aggregation}) arise as intermediate or specialized results, with 
their proofs interleaved into the main argument.

We first establish the $\delta$-PAC correctness 
guarantee under $\alpha=1$. Then, for a fixed 
admissible chain $\sigma\in\mathfrak S$, we develop 
the finite-sample analysis directly for a general 
chain with $K-1$ elimination transitions. We aggregate 
the per-chain bounds over $\sigma\in\mathfrak S$, 
control the wrong-elimination complement, and conclude 
with a unified expectation bound for all $\alpha\in(0,1]$.

\begin{theorem}[$\delta$-PAC correctness]
\label{thm:delta-pac-correctness}
Assume $\alpha=1$, so that 
$\beta_{\mathrm{elim}}(t,\delta) 
= \beta_{\mathrm{stop}}(t,\delta)$. 
Then $\Pr(\hat h(\tau) \neq h^*) \le \delta$.
\end{theorem}

\begin{proof}
\[
\begin{aligned}
&\Pr(\hat h(\tau) \neq h^*)\\
&\overset{(a)}{\le}
\Pr\bigl(\exists\, g \neq h^* : 
G_\tau(g) = \emptyset\bigr)\\
&\overset{(b)}{\le}
\Pr\bigl(\exists\, g \neq h^*,\,
\exists\, t \le \tau :
Z_t(g, h^*) \ge 
\beta_{\mathrm{elim}}(t,\delta)\bigr)\\
&\overset{(c)}{\le}
\Pr\bigl(\exists\, g \neq h^*,\,
\exists\, t \ge 1 :
Z_t(g, h^*) \ge 
\beta_{\mathrm{elim}}(t,\delta)\bigr)\\
&\overset{(d)}{=}
\Pr\bigl(\exists\, g \neq h^*,\,
\exists\, t \ge 1 :
Z_t(g, h^*) \ge 
\beta_{\mathrm{stop}}(t,\delta)\bigr)\\
&\overset{(e)}{\le} \delta,
\end{aligned}
\]
where (a) follows from the stopping rule 
(the algorithm outputs $g$ only when 
$G_\tau(g) = \emptyset$); 
(b) holds because $G_\tau(g) = \emptyset$ 
requires $h^*$ to have been eliminated from 
$G_t(g)$ at some round $t \le \tau$, which 
needs $Z_t(g,h^*) \ge \beta_{\mathrm{elim}}(t,\delta)$; 
(c) relaxes $t \le \tau$ to $t \ge 1$; 
(d) uses $\alpha = 1$; and 
(e) is the time-uniform concentration inequality 
used to calibrate $\beta_{\mathrm{stop}}$ 
{\normalfont(Kaufmann \& Koolen, Cor.~10~\cite{c19}; 
Lemma~2.3, Tirinzoni \& Degenne~\cite{c9})}.
\end{proof}

\begin{remark}[Error guarantee under $\alpha<1$]
\label{rem:error-alpha-less-one}
When $\alpha<1$, the elimination threshold
$\beta_{\mathrm{elim}}(t,\delta)=\alpha L+b\log t+c$ is 
strictly below $\beta_{\mathrm{stop}}(t,\delta)$. Since
$\mathcal{E}_{\mathrm{corr}}^c=\{\hat h(\tau)\neq h^*\}$
(by Definition~5 and the stopping rule), steps (a)--(c) 
of Theorem~\ref{thm:delta-pac-correctness} still apply, 
giving
\[
\begin{aligned}
 \Pr(\mathcal E_{\mathrm{corr}}^c)
&\le \Pr\bigl(\exists g\neq h^*,\ \exists s\ge 1:\
 Z_s(g,h^*)\ge\beta_{\mathrm{elim}}(s,\delta)\bigr)\\
&\le c_{\mathrm{elim}}\,\delta^\alpha,   
\end{aligned}
\]
where $c_{\mathrm{elim}}>0$ absorbs the union-bound 
factor $|\mathcal H|-1$ and constants from the 
time-uniform inequality applied with the threshold 
shifted by $(1-\alpha)L$. In particular, 
$\Pr(\hat h(\tau)\neq h^*)\le c_{\mathrm{elim}}\,
\delta^\alpha$.
\end{remark}

\subsubsection{Basic ingredients and proof inputs}
\label{app:ingredients}

\begin{lemma}[Pairwise domination of the MLE champion]
\label{lem:champion-pairwise}
If $\hat h(t)\neq h^*$, then there exists $g\neq h^*$ 
such that $Z_t(h^*,g)\le 0$.
\end{lemma}

\begin{proof}
Since $\hat h(t)$ maximizes the likelihood, 
$\sum_{i=1}^t \log p_{A_i}(O_i\mid \hat h(t))
\ge \sum_{i=1}^t \log p_{A_i}(O_i\mid h^*)$, which 
gives $Z_t(h^*,\hat h(t))\le 0$.
\end{proof}

\begin{lemma}[Forced exploration; Garivier \& 
Kaufmann, Lemma~7~\cite{c8}]
\label{lem:forced-exploration}
Under the C-Tracking selection rule, for all 
$a\in\mathcal A$ and all $t\ge 1$,
\[
N_a(t)\ge \sqrt{t+|\mathcal A|^2}-2|\mathcal A|.
\]
\end{lemma}

\begin{lemma}[C-Tracking cumulative deviation; Garivier 
\& Kaufmann, Lemma~7~\cite{c8}]
\label{lem:tracking}
Let $W_a^{\mathrm{tar}}(t):=\sum_{s=1}^t u_a(s)$ 
denote the cumulative tracking target under the 
active-set-aware oracle $u(s)\in\Delta(\mathcal A)$, 
and let $w_E(t):=N(t)/t$ denote the empirical action 
frequency. The C-Tracking selection rule satisfies, 
for all $a\in\mathcal A$ and all $t\ge 1$,
\[
\max_{a\in\mathcal A}\bigl|N_a(t)-W_a^{\mathrm{tar}}(t)
\bigr|\le |\mathcal A|\bigl(1+\sqrt t\bigr).
\]
Consequently, with $c_{\mathrm{tr}}:=2|\mathcal A|^2$,
\[
\|w_E(t)-\bar u(t)\|_1\le \frac{c_{\mathrm{tr}}}
{\sqrt t},\quad \forall t\ge 1.
\]
\end{lemma}

\begin{lemma}[Martingale concentration]
\label{lem:martingale-concentration}
For $g\in S_0$, define
\[
\Delta M_r^g
:=
\log\frac{p_{A_r}(O_r\mid h^*)}{p_{A_r}(O_r\mid g)}
-d_{A_r}(h^*,g),\
M_s^g:=\sum_{r=1}^s \Delta M_r^g.
\]
Assume there exists a constant $v_M>0$ such that, for 
all $r\ge 1$, all $g\in S_0$, and all $\lambda\in
\mathbb R$,
\[
\mathbb E_{h^*}\!\left[
\exp\!\bigl(\lambda \Delta M_r^g\bigr)
\,\middle|\,\mathcal F_{r-1}
\right]
\le
\exp\!\left(\frac{\lambda^2 v_M}{2}\right)
\ \text{a.s.}
\]
Then for every nonempty $S\subseteq S_0$ there exist 
constants $\eta_S,c_{M,S}>0$ such that, for every 
$t\ge 3$,
\[
G_t^{M,S}
:=
\Bigl\{
\forall g\in S,\ \forall s\le t:\ |M_s^g|
\le \eta_S\sqrt{s\log(e+t)}
\Bigr\}
\]
satisfies $\Pr\bigl((G_t^{M,S})^c\bigr)\le c_{M,S}/t^3$.
\end{lemma}

\begin{proof}
Fix $g\in S$. By the tower property and the 
conditional sub-Gaussian assumption, for every 
$\lambda>0$, $\mathbb E_{h^*}[e^{\lambda M_s^g}]
\le e^{\lambda^2 v_M s/2}$ (iterating). By Markov's 
inequality and optimizing $\lambda=x/(v_M s)$,
$\Pr(M_s^g\ge x)\le\exp(-x^2/(2v_M s))$. The same for 
$-M_s^g$ gives $\Pr(|M_s^g|\ge x)\le 2\exp(-x^2/(2
v_M s))$. Taking $x=\sqrt{10v_M s\log(e+t)}\ge
\sqrt{10v_M s\log t}$ (for $t\ge 3$) yields 
$\Pr(|M_s^g|\ge\sqrt{10v_M s\log(e+t)})\le 2t^{-5}$.
Union bounds over $s\le t$ and $g\in S$ give 
$\Pr((G_t^{M,S})^c)\le 2|S|t^{-4}\le c_{M,S}/t^3$ 
with $\eta_S=\sqrt{10v_M}$, $c_{M,S}=2|S|$.
\end{proof}

\begin{remark}[Gaussian observation models]
\label{rem:gaussian-mgf}
If $O_t\mid(A_t=a, h^*=h)\sim\mathcal N(\mu_{a,h}, 
\sigma_a^2)$ with common variance $\sigma_a^2$ across 
hypotheses for each action $a$ (the setting of 
Section~V), then $\Delta M_t^g\mid(\mathcal F_{t-1}, 
A_t=a)$ is Gaussian with mean zero and variance 
$2 d_a(h^*,g)$. Hence the conditional sub-Gaussian 
assumption of 
Lemma~\ref{lem:martingale-concentration} holds with
\[
v_M=2\max_{a\in\mathcal A}\max_{g\in S_0}d_a(h^*,g),
\]
linking the appendix assumption to the experimental 
setup of Section~V.
\end{remark}

\begin{lemma}[Elimination trigger on the fixed chain]
\label{lem:elim-trigger}
On the chain event $\mathcal{E}$, if at some round 
$t$, $\hat h(t)=h^*$ and $Z_t(h^*,g)\ge 
\beta_{\mathrm{elim}}(t,\delta)$ for all 
$g\in S_{r-1}\setminus S_r$, then the $r$-th 
elimination has already been triggered: $T_r\le t$. 
Here $r$ ranges over $\{1,\ldots,K-1\}$; in particular, 
for $r=K-1$, the active set $G_t(h^*)$ becomes empty 
and the stopping rule (Algorithm~1, line~12) fires: 
$\tau\le t$.
\end{lemma}

\begin{proof}
On $\mathcal{E}$, the only possible transition out of 
$S_{r-1}$ removes exactly $S_{r-1}\setminus S_r$. If 
all these opponents cross the threshold while 
$\hat h(t)=h^*$, the update rule triggers the 
transition, giving $T_r\le t$. When $r=K-1$, 
$S_r=\emptyset$, so $G_t(h^*)=\emptyset$ and the 
stopping rule fires, yielding $\tau\le T_{K-1}\le t$.
\end{proof}

\subsubsection{Basic geometric facts}
\label{app:geometric}

\begin{lemma}[Set monotonicity]
\label{lem:set-monotonicity}
For every $S'\subseteq S\subseteq\mathcal H
\setminus\{h^*\}$ and every $w\in\Delta(\mathcal A)$,
$f_{S'}(w)\ge f_S(w)$. In particular, $D_k\ge D_{k-1}$ 
for every $k=1,\dots,K-2$.
\end{lemma}

\begin{proof}
Since $S'\subseteq S$, the minimization defining 
$f_{S'}(w)$ is over a smaller set than $f_S(w)$, so 
$f_{S'}(w)\ge f_S(w)$. Applying with $S'=S_k$, 
$S=S_{k-1}$, $w=w^{(k-1)}$ gives 
$D_k=f_{S_k}(w^{(k)})\ge f_{S_k}(w^{(k-1)})\ge 
f_{S_{k-1}}(w^{(k-1)})=D_{k-1}$.
\end{proof}

\begin{lemma}[Concavity]
\label{lem:concavity}
For every nonempty $S\subseteq\mathcal H\setminus
\{h^*\}$, $f_S(w)=\min_{g\in S}\sum_a w_a d_a(h^*,g)$ 
is concave on $\Delta(\mathcal A)$.
\end{lemma}

\begin{lemma}[Lipschitz continuity]
\label{lem:lipschitz}
For any nonempty $S\subseteq\mathcal H\setminus\{h^*\}$, 
define $L_S:=\max_{g\in S}\max_a d_a(h^*,g)$. Then for 
all $w,w'\in\Delta(\mathcal A)$,
$|f_S(w)-f_S(w')|\le L_S\|w-w'\|_1$.
\end{lemma}

\begin{proof}
$|\min_i x_i-\min_i y_i|\le\max_i|x_i-y_i|$ applied 
to $f_g(w):=\sum_a w_a d_a(h^*,g)$.
\end{proof}

\subsubsection{Champion stabilization}
\label{app:champion}

\begin{lemma}[Champion stabilization tail]
\label{lem:Tch-tail}
Recall $T_{\mathrm{ch}}:=\inf\{s\ge 1:\hat h(r)=h^*\text{ for all }r\ge s\}$.
There exist constants $B_{\mathrm{ch}},C_{\mathrm{ch}}>0$ such that
\[
\Pr(T_{\mathrm{ch}}>u)\le B_{\mathrm{ch}}\,u\,e^{-C_{\mathrm{ch}}\sqrt u},
\quad \forall\, u\ge 1.
\]
\end{lemma}

\begin{proof}
Fix $g\neq h^*$.
For $a\in\mathcal A$, let
\[
\rho_{a,g}:=\int \sqrt{p_a(x\mid h^*)\,p_a(x\mid g)}\,dx,
\quad
c_{a,g}:=-\log\rho_{a,g}\ge 0.
\]
By identifiability, for every $g\neq h^*$ there exists at least one $a$ such
that $c_{a,g}>0$.

Conditioning on the action sequence up to time $t$ gives
\[
\begin{aligned}
\mathbb E_{h^*}\!\left[e^{-Z_t(h^*,g)/2}\,\middle|\,A_1,\dots,A_t\right]
&= \prod_{s=1}^t \rho_{A_s,g}\\
&= \exp\!\left(-\sum_{a\in\mathcal A} c_{a,g}N_a(t)\right).
\end{aligned}
\]
Taking expectations and using the forced-exploration lower bound
$N_a(t)\ge c_{\mathrm{fe}}\sqrt t-b_{\mathrm{fe}}$, we obtain
\[
\begin{aligned}
\mathbb E_{h^*}\!\left[e^{-Z_t(h^*,g)/2}\right]
&\le
\exp\!\left(-\sum_a c_{a,g}(c_{\mathrm{fe}}\sqrt t-b_{\mathrm{fe}})\right)\\
&=
B_g e^{-C_g\sqrt t},
\end{aligned}
\]
for suitable constants $B_g,C_g>0$.

Markov's inequality now gives
\[
\begin{aligned}
\Pr\bigl(Z_t(h^*,g)\le 0\bigr)
&= \Pr\bigl(e^{-Z_t(h^*,g)/2}\ge 1\bigr)\\
&\le \mathbb E\!\left[e^{-Z_t(h^*,g)/2}\right]\\
&\le B_g e^{-C_g\sqrt t}.
\end{aligned}
\]
By Lemma~\ref{lem:champion-pairwise}, if $\hat h(t)\neq h^*$ then there
exists some $g\neq h^*$ such that $Z_t(h^*,g)\le 0$.
Therefore, taking a union bound over $g\neq h^*$, there exist constants
$B,C>0$ such that
\[
\Pr(\hat h(t)\neq h^*)
\le
B e^{-C\sqrt t}.
\]

Finally,
\[
\begin{aligned}
\Pr(T_{\mathrm{ch}}>u)
&= \Pr(\exists t\ge u:\ \hat h(t)\neq h^*)\\
&\le \sum_{t=u}^{\infty} \Pr(\hat h(t)\neq h^*)\\
&\le \sum_{t=u}^{\infty} B e^{-C\sqrt t}.
\end{aligned}
\]
Using an integral comparison,
\[
\begin{aligned}
\sum_{t=u}^{\infty} e^{-C\sqrt t}
\le
\int_{u-1}^{\infty} e^{-C\sqrt x}\,dx
&=
\int_{\sqrt{u-1}}^{\infty} 2y e^{-Cy}\,dy\\
&\le
B' u e^{-C'\sqrt u}
\end{aligned}
\]
for suitable constants $B',C'>0$.
Absorbing constants completes the proof.
\end{proof}

\begin{definition}[Time-indexed champion event]
\label{def:Etch}
Fix $\kappa_{\mathrm{ch}}>0$ and define, for $t\ge 3$,
\[
m_{\mathrm{ch}}(t):=\lceil\kappa_{\mathrm{ch}}
\log^2 t\rceil,\ \ 
\overline m_{\mathrm{ch}}(t):=1+\kappa_{\mathrm{ch}}
\log^2(e+t),
\]
\[
E_t^{\mathrm{ch}}:=\{T_{\mathrm{ch}}\le 
m_{\mathrm{ch}}(t)\}.
\]
The smooth envelope satisfies $m_{\mathrm{ch}}(t)
\le\overline m_{\mathrm{ch}}(t)$ for all $t\ge 3$. 
Define the chain-specific good event
\[
\mathcal{G}_t := E_t^{\mathrm{ch}} \cap 
\bigcap_{k=0}^{K-2} G_t^{M,S_k}.
\]
\end{definition}

\begin{proposition}[Summable failure probability 
for the good event]
\label{prop:good-event-summable}
The quantity 
\[
C_{\mathrm{bad}}:=\sum_{t=3}^\infty\Pr(
\mathcal G_t^c)
\]
is finite and $\delta$-independent.
\end{proposition}

\begin{proof}
Recall $\mathcal G_t:=E_t^{\mathrm{ch}}\cap 
G_t^{M,S_0}$, so
\[
\Pr(\mathcal G_t^c)\le\Pr((E_t^{\mathrm{ch}})^c)
+\Pr((G_t^{M,S_0})^c)
\]
by union bound. We bound each term separately.

\emph{Champion event.} By 
Lemma~\ref{lem:Tch-tail}, for every 
$s\ge 1$,
\[
\Pr(T_{\mathrm{ch}}>s)\le B_{\mathrm{ch}}\,s\,
e^{-C_{\mathrm{ch}}\sqrt s},
\]
where $B_{\mathrm{ch}},C_{\mathrm{ch}}>0$ are 
$\delta$-independent chain constants.

Taking $s=m_{\mathrm{ch}}(t)=\lceil\kappa_
{\mathrm{ch}}\log^2 t\rceil$,
\[
\Pr((E_t^{\mathrm{ch}})^c)=\Pr(T_{\mathrm{ch}}
>m_{\mathrm{ch}}(t))\le B_{\mathrm{ch}}\,
m_{\mathrm{ch}}(t)\,e^{-C_{\mathrm{ch}}
\sqrt{m_{\mathrm{ch}}(t)}}.
\]

Since $\sqrt{m_{\mathrm{ch}}(t)}\ge\sqrt{\kappa_
{\mathrm{ch}}}\log t$, we have $e^{-C_{\mathrm
{ch}}\sqrt{m_{\mathrm{ch}}(t)}}\le t^{-C_{\mathrm
{ch}}\sqrt{\kappa_{\mathrm{ch}}}}$. Choosing 
$\kappa_{\mathrm{ch}}\ge 9/C_{\mathrm{ch}}^2$ so 
that $C_{\mathrm{ch}}\sqrt{\kappa_{\mathrm{ch}}}
\ge 3$, and using $m_{\mathrm{ch}}(t)\le 
2\kappa_{\mathrm{ch}}\log^2 t$ for $t\ge 3$,
\[
\Pr((E_t^{\mathrm{ch}})^c)\le\frac{2B_{\mathrm
{ch}}\kappa_{\mathrm{ch}}\log^2 t}{t^3},
\]
which is summable: $\sum_{t=3}^\infty\Pr
((E_t^{\mathrm{ch}})^c)\le C_1<\infty$.

\emph{Martingale event.} By 
Lemma~\ref{lem:martingale-concentration} with sub-Gaussian parameter $v_M$,
\[
\Pr((G_t^{M,S})^c)\le 2|S|t^{-4}\le c_{M,S}/t^3,
\]
so $\sum_{t=3}^\infty\Pr((G_t^{M,S_0})^c)\le 
C_2<\infty$.

Combining, $C_{\mathrm{bad}}=\sum_{t=3}^\infty\Pr(
\mathcal G_t^c)<\infty$, depending only on 
$\kappa_{\mathrm{ch}}$, $B_{\mathrm{ch}}$, 
$C_{\mathrm{ch}}$, $\eta_{S_0}$, $|S_0|$, $v_M$, 
none of which depend on $\delta$.
\end{proof}

\begin{remark}[Relative ordering of 
$T_{\mathrm{ch}}$, $T_1,\ldots,T_{K-1}$]
\label{rem:ordering}
The proof does \emph{not} assume any deterministic 
ordering among the champion stabilization time 
$T_{\mathrm{ch}}$ and the elimination times 
$T_1,\ldots,T_{K-1}$ (recall $T_{K-1}=\tau$ since 
the final transition to $S_{K-1}=\emptyset$ coincides 
with the stopping time).

\begin{enumerate}
\item If $T_{\mathrm{ch}}\le T_1\le\cdots\le T_{K-1}$, 
the champion stabilizes first, eliminations occur in 
order, and the argument is the most direct.

\item If $T_r<T_{\mathrm{ch}}\le T_{K-1}$ for some 
$r$, then one or more eliminations occur before the 
champion has permanently stabilized. This is not a 
problem on the chain event $\mathcal{E}$: the active 
set still shrinks correctly, and the temporary 
mismatch caused by a wrong champion is captured by 
the champion-induced deficit $\Gamma_{\mathrm{ch}}(t)$.
\end{enumerate}
The proof strategy is to show that for every 
sufficiently large deterministic time $t$, the event 
$\mathcal{E}\cap\mathcal{G}_t$ forces $\tau\le t$, 
using the recursive surrogate certificates 
$\bar T_1,\ldots,\bar T_{K-1}$.
\end{remark}

\begin{definition}[Champion-induced deficit]
\label{def:champion-deficit}
Define $\Delta_{\mathrm{ch}}(s):=D_{k(s)}
-f_{S_{k(s)}}(u(s))$, $\Gamma_{\mathrm{ch}}(t):=
\sum_{s=1}^t \Delta_{\mathrm{ch}}(s)$, where 
$k(s):=\max\{j:T_j\le s\}$ denotes the stage index 
at time $s$ (with $T_0:=0$).
\end{definition}

\begin{lemma}[Pathwise control of 
$\Gamma_{\mathrm{ch}}$]
\label{lem:Gamma-pathwise}
On $\mathcal E$:
(1) $\hat h(s)=h^*\Rightarrow\Delta_{\mathrm{ch}}(s)
=0$;
(2) $0\le\Delta_{\mathrm{ch}}(s)\le D_{k(s)}
\mathbf 1\{\hat h(s)\neq h^*\}$;
(3) $B_t:=\sum_s\mathbf 1\{\hat h(s)\neq h^*\}\le 
T_{\mathrm{ch}}$;
(4) On $\mathcal E\cap E_t^{\mathrm{ch}}$, for any 
$r\in\{1,\ldots,K-1\}$ and $m_{\mathrm{ch}}(t)\le s
<T_r$: $\Gamma_{\mathrm{ch}}(s)\le D_{r-1}
m_{\mathrm{ch}}(t)$. 
\end{lemma}

\begin{proof}
(1)--(3) by oracle definition, KL non-negativity, and 
counting wrong-champion rounds. (4): on 
$E_t^{\mathrm{ch}}$, wrong-champion $s'\le 
T_{\mathrm{ch}}\le m_{\mathrm{ch}}(t)\le s<T_r$, so 
$k(s')\le r-1$ and $D_{k(s')}\le D_{r-1}$ by 
Lemma~\ref{lem:set-monotonicity}.
\end{proof}

\subsubsection{Information lower bounds and recursive 
surrogate certificates}
\label{app:information-bounds}

\paragraph{Stage-$r$ drift lower bound.}

\begin{proposition}[Average-target lower bound before 
stage $r$]
\label{prop:stage-average}
Fix $r\in\{1,\ldots,K-1\}$, $t\ge 3$. Suppose $T_{r-1}\le s<T_r$ with $m_{\mathrm{ch}}(t)\le s\le t$ (by convention $T_0:=0$, so for $r=1$ the condition 
reduces to $m_{\mathrm{ch}}(t)\le s<T_1$). On $\mathcal E\cap E_t^{\mathrm{ch}}$,
\[
f_{S_{r-1}}(\bar u(s))\ge D_{r-1}
-\tfrac{\Pi_{r-1}+D_{r-1}m_{\mathrm{ch}}(t)}{s},
\]
where $\Pi_{r-1}:=\sum_{j=1}^{r-1}T_j(D_j-D_{j-1})$ (with $\Pi_0:=0$). Consequently,
\[
f_{S_{r-1}}(w_E(s))\ge D_{r-1}
-\tfrac{\Pi_{r-1}+D_{r-1}m_{\mathrm{ch}}(t)}{s}
-\tfrac{L_{S_{r-1}}c_{\mathrm{tr}}}{\sqrt s}.
\]
\end{proposition}

\begin{proof}
By concavity and Jensen's inequality, $f_{S_{r-1}}(\bar u(s))\ge s^{-1}
\sum_{s'=1}^s f_{S_{r-1}}(u(s'))$. Set monotonicity 
gives $f_{S_{r-1}}(u(s'))\ge D_{k(s')}
-\Delta_{\mathrm{ch}}(s')$ for $s'<T_r$. Abel 
summation with $T_0=0$ and $T_{r-1}\le s<T_r$ gives 
$\sum_{s'=1}^s D_{k(s')}=sD_{r-1}-\Pi_{r-1}$. 
Lemma~\ref{lem:Gamma-pathwise} bounds 
$\Gamma_{\mathrm{ch}}(s)\le D_{r-1}m_{\mathrm{ch}}(t)$. 
Combining gives the first claim; Lipschitz and 
C-Tracking give the second.
\end{proof}

\begin{proposition}[Stage-$r$ evidence lower bound]
\label{prop:stage-evidence}
Under the conditions of 
Proposition~\ref{prop:stage-average}, on 
$\mathcal E\cap\mathcal G_t$,
\[
Z_s(h^*,g)\ge sD_{r-1}-\Pi_{r-1}-R_{r-1}(s,t)
\]
for all $g\in S_{r-1}\setminus S_r$, where $R_{r-1}$ 
is as in Definition~\ref{def:surrogates} below. In 
particular, for $r=K-1$, the bound holds for all 
$g\in S_{K-2}$ (the final non-empty active set).
\end{proposition}

\begin{proof}
$Z_s(h^*,g)=\sum_a N_a(s)d_a(h^*,g)+M_s^g\ge 
sf_{S_{r-1}}(w_E(s))+M_s^g$. Combine with 
Proposition~\ref{prop:stage-average}, $M_s^g\ge 
-\eta_{r-1}\sqrt{s\log(e+t)}$ on $G_t^{M,S_{r-1}}$, 
and $D_{r-1}\overline m_{\mathrm{ch}}(t)\ge 
D_{r-1}m_{\mathrm{ch}}(t)$.
\end{proof}

\paragraph{Recursive surrogate certificates.}

\begin{definition}[Warm-up time]
\label{def:T-warm}
Define the $\delta$-independent warm-up time
\[
T_{\mathrm{warm}}:=\min\{t\ge 3:\overline m_{\mathrm{ch}}(t)\le t\}.
\]
All subsequent analyses implicitly assume $t\ge T_{\mathrm{warm}}$.
\end{definition}

\begin{definition}[Deterministic surrogates]
\label{def:surrogates}
Fix $t\ge T_{\mathrm{warm}}$. Set $\bar T_0(t,\delta)
:=0$, $\bar\Pi_0(t,\delta):=0$. Define the remainder 
with smooth envelope
\[
\begin{aligned}
R_{r-1}(s,t):=\ &D_{r-1}\overline m_{\mathrm{ch}}(t)
+L_{S_{r-1}}c_{\mathrm{tr}}\sqrt s\\
&+\eta_{r-1}\sqrt{s\log(e+t)}.
\end{aligned}
\]
For $r=1,\ldots,K-1$:

\noindent\textbf{(a)} $\bar\Pi_{r-1}(t,\delta)
:=\sum_{j=1}^{r-1}\bar T_j(t,\delta)(D_j-D_{j-1})$.

\noindent\textbf{(b)} $\Gamma_{r-1}(s,t):=sD_{r-1}
-\bar\Pi_{r-1}(t,\delta)-R_{r-1}(s,t)$, and
\[
\bar T_r(t,\delta):=\inf\Bigl\{\, s\in\mathbb N \,
\Big|\,
\substack{
 \overline m_{\mathrm{ch}}(t)\le s\le t-1,\\
 s\ge\bar T_{r-1}(t,\delta)+1,\\
 \beta_{\mathrm{elim}}(s,\delta)\le\Gamma_{r-1}(s,t)
}
\Bigr\}.
\]
\end{definition}


\begin{proposition}[Recursive transition certificates]
\label{prop:recursive-certificates}
Fix $t\ge T_{\mathrm{warm}}$ with 
$\bar T_r(t,\delta)\le t-1$ for all $r\in\{1,\ldots,
K-1\}$. On $\mathcal E\cap\mathcal G_t\cap\{\tau>t\}$, 
$T_r\le\bar T_r(t,\delta)$ for $r=1,\ldots,K-1$.
\end{proposition}

\begin{proof}
Induction on $r$.

\emph{Base ($r=1$).} Suppose $T_1>\bar T_1$; set 
$s:=\bar T_1$. Since $s\ge\overline m_{\mathrm{ch}}(t)
\ge m_{\mathrm{ch}}(t)$ and $E_t^{\mathrm{ch}}$ holds, 
$\hat h(s)=h^*$; since $T_1>s$, 
Proposition~\ref{prop:stage-evidence} with $\Pi_0=0$ 
gives $Z_s(h^*,g)\ge sD_0-R_0(s,t)\ge 
\beta_{\mathrm{elim}}(s,\delta)$ for all $g\in 
S_0\setminus S_1$. Lemma~\ref{lem:elim-trigger} 
yields $T_1\le s$, contradicting $T_1>s$.

\emph{Inductive step.} Assume $T_j\le\bar T_j$ for 
$j\le r-1$. Suppose $T_r>\bar T_r$; set $s:=\bar T_r$. 
Then $s\ge\bar T_{r-1}+1\ge T_{r-1}+1>T_{r-1}$ (by 
induction) and $T_r>s$ (for contradiction), so 
$T_{r-1}\le s<T_r$ (the condition of 
Proposition~\ref{prop:stage-average}); $s\ge
\overline m_{\mathrm{ch}}(t)$ and $\hat h(s)=h^*$. By 
Lemma~\ref{lem:deficit-telescoping}, 
$\Pi_{r-1}\le\bar\Pi_{r-1}$. 
Proposition~\ref{prop:stage-evidence} gives 
$Z_s(h^*,g)\ge sD_{r-1}-\bar\Pi_{r-1}-R_{r-1}(s,t)
\ge\beta_{\mathrm{elim}}(s,\delta)$ for all 
$g\in S_{r-1}\setminus S_r$. 
Lemma~\ref{lem:elim-trigger} yields $T_r\le s$, 
contradicting $T_r>s$. For $r=K-1$, the triggering 
condition uses $g\in S_{K-2}=S_{r-1}\setminus S_r$ 
(since $S_r=S_{K-1}=\emptyset$), and 
Lemma~\ref{lem:elim-trigger} additionally gives 
$\tau\le T_{K-1}\le s$.
\end{proof}

\begin{lemma}[Telescoping form of the deficit]
\label{lem:deficit-telescoping}
$\Pi_m=\sum_{j=1}^m T_j(D_j-D_{j-1})$ and 
$\bar\Pi_m=\sum_{j=1}^m\bar T_j(D_j-D_{j-1})$. If 
$T_j\le\bar T_j$ and $D_j-D_{j-1}\ge 0$, then 
$\Pi_m\le\bar\Pi_m$.
\end{lemma}

\paragraph{Explicit bound for recursive certificates}
The recursive certificates are \emph{implicit}. For each stage $r$, the surrogate time
$\bar T_r(t,\delta)$ is defined as the earliest integer $s$
for which the deterministic lower bound on the stage-$r$
evidence already exceeds the elimination threshold.
This is conceptually clean, but it is not yet in a form that
can be inserted into the tail-sum bound for
$\mathbb E[\tau\mathbf 1_{\mathcal E_\sigma}]$.

\begin{proposition}[Explicit bound for $\bar T_r$]
\label{prop:Tr-explicit}
There exist chain-dependent constants $A_1,A_2,A_3>0$, 
independent of $\delta$ and $r$, such that with
\[
\tilde s_1(t,\delta):=\left\lceil
\begin{aligned}
&\tfrac{\alpha L+b\log(e+t)}{D_0}
+A_1\sqrt{L\log(e+t)}\\
&\ +A_2\log^2(e+t)+A_3
\end{aligned}
\right\rceil,
\]
$\tilde s_r(t,\delta):=\tilde s_1(t,\delta)+(r-1)$: 
for every $t\ge T_{\mathrm{warm}}$ and every 
$r\in\{1,\ldots,K-1\}$ with $\tilde s_r(t,\delta)
\le t-1$, $\bar T_r(t,\delta)\le\tilde s_r(t,\delta)$.
\end{proposition}

\begin{proof}
The proof is by induction on $r$.

The key point is that $\bar T_r(t,\delta)$ is defined as
the \emph{earliest} integer $s$ satisfying the three constraints
in Definition~\ref{def:surrogates}: champion stabilization,
stage ordering, and threshold crossing by the deterministic
evidence lower bound. Therefore, to prove
$\bar T_r(t,\delta)\le \tilde s_r(t,\delta)$, it is enough to
verify that the explicit candidate $s=\tilde s_r(t,\delta)$
satisfies those defining constraints.

\emph{Base case ($r=1$).}
Set $s:=\tilde s_1(t,\delta)$.
We verify the defining conditions of $\bar T_1(t,\delta)$.

First, the champion-stabilization condition holds:
\[
\overline m_{\mathrm{ch}}(t)=1+\kappa_{\mathrm{ch}}\log^2(e+t)
\le A_2\log^2(e+t)+A_3 \le s
\]
provided $A_2\ge \kappa_{\mathrm{ch}}, A_3\ge 1.$

Second, the stage-order condition is vacuous at $r=1$
because $\bar T_0(t,\delta)=0$, and the requirement
$s\le t-1$ is part of the statement.

Third, we verify the threshold-crossing condition
\[
\beta_{\mathrm{elim}}(s,\delta)\le \Gamma_0(s,t)
= sD_0-R_0(s,t).
\]
Since $s\le t$, we have
\[
\beta_{\mathrm{elim}}(s,\delta)
=
\alpha L+b\log s+c
\le
\alpha L+b\log(e+t)+c.
\]
On the other hand, by the definition of $\tilde s_1$,
\[
\begin{aligned}
sD_0\ge {}& \alpha L+b\log(e+t)
+D_0A_1\sqrt{L\log(e+t)}\\
& +D_0A_2\log^2(e+t)
+D_0A_3.
\end{aligned}
\]
By Remark~\ref{rem:R-uniform-envelope},
\[
R_0(s,t)\le
B_1\sqrt{L\log(e+t)}
+B_2\log^2(e+t)+B_3.
\]
Hence
\[
\begin{aligned}
sD_0-R_0(s,t)\ge& {} \alpha L+b\log(e+t)
\\&+(D_0A_1-B_1)\sqrt{L\log(e+t)}\\
&+(D_0A_2-B_2)\log^2(e+t)
+(D_0A_3-B_3).
\end{aligned}
\]
Therefore, if
\[
A_1\ge \frac{B_1}{D_0},\quad
A_2\ge \max\!\left\{\kappa_{\mathrm{ch}},\frac{B_2}{D_0}\right\},\quad
A_3\ge \max\!\left\{1,\frac{B_3+c}{D_0}\right\},
\]
then
\[
sD_0-R_0(s,t)\ge \beta_{\mathrm{elim}}(s,\delta).
\]
Thus $s=\tilde s_1(t,\delta)$ satisfies all defining
constraints of $\bar T_1(t,\delta)$, which proves
\[
\bar T_1(t,\delta)\le \tilde s_1(t,\delta).
\]

\medskip
\noindent
\emph{Inductive step.}
Assume that
\[
\bar T_j(t,\delta)\le \tilde s_j(t,\delta),
\quad j=1,\ldots,r-1.
\]
We prove that
\[
\bar T_r(t,\delta)\le \tilde s_r(t,\delta).
\]

Set
$
s:=\tilde s_r(t,\delta)=\tilde s_{r-1}(t,\delta)+1.
$
This is the natural choice because the definition of
$\bar T_r(t,\delta)$ requires
\[
s\ge \bar T_{r-1}(t,\delta)+1,
\]
and the induction hypothesis gives
\[
\bar T_{r-1}(t,\delta)\le \tilde s_{r-1}(t,\delta).
\]
Hence
$
s=\tilde s_{r-1}(t,\delta)+1\ge \bar T_{r-1}(t,\delta)+1,
$
so the stage-order constraint is automatically satisfied.

The stabilization condition
$\overline m_{\mathrm{ch}}(t)\le s$ also follows immediately
from the base case because $\tilde s_r\ge \tilde s_1$.

It remains to verify the threshold-crossing condition
\[
\beta_{\mathrm{elim}}(s,\delta)\le \Gamma_{r-1}(s,t)
=
sD_{r-1}-\bar\Pi_{r-1}(t,\delta)-R_{r-1}(s,t).
\]

The crucial observation is that the accumulated deficit
$\bar\Pi_{r-1}(t,\delta)$ can be controlled using the induction
hypothesis and Lemma~\ref{lem:deficit-telescoping}:
\[
\begin{aligned}
\bar\Pi_{r-1}(t,\delta)
&=
\sum_{j=1}^{r-1}\bar T_j(t,\delta)(D_j-D_{j-1})\\
&\le
\sum_{j=1}^{r-1}\tilde s_j(t,\delta)(D_j-D_{j-1}).    
\end{aligned}
\]
Since $\tilde s_j(t,\delta)\le \tilde s_{r-1}(t,\delta)$
for all $j\le r-1$, we obtain
\[
\begin{aligned}
\bar\Pi_{r-1}(t,\delta)
&\le
\tilde s_{r-1}(t,\delta)\sum_{j=1}^{r-1}(D_j-D_{j-1})\\
&=
\tilde s_{r-1}(t,\delta)(D_{r-1}-D_0).
\end{aligned}
\]
Therefore
\[
\begin{aligned}
sD_{r-1}-\bar\Pi_{r-1}(t,\delta)
&\ge
(\tilde s_{r-1}+1)D_{r-1}
-\tilde s_{r-1}(D_{r-1}-D_0)\\
&=
\tilde s_{r-1}D_0 + D_{r-1}\\
&\ge
\tilde s_1D_0 + D_{r-1}.
\end{aligned}
\]

Using again Remark~\ref{rem:R-uniform-envelope}, and since $\tilde s_1D_0$ was constructed precisely to dominate
\[
\alpha L+b\log(e+t)
+B_1\sqrt{L\log(e+t)}
+B_2\log^2(e+t)+B_3+c,
\]
the same constant choice as in the base case yields
\[
sD_{r-1}-\bar\Pi_{r-1}(t,\delta)-R_{r-1}(s,t)
\ge
\beta_{\mathrm{elim}}(s,\delta).
\]
Thus $s=\tilde s_r(t,\delta)$ satisfies every defining
constraint of $\bar T_r(t,\delta)$, proving
\[
\bar T_r(t,\delta)\le \tilde s_r(t,\delta).
\]
This completes the induction.
\end{proof}

\begin{remark}[Uniform envelope for $R_{r-1}$]
\label{rem:R-uniform-envelope}
For every $s\le\tilde s_r(t,\delta)$ (with 
$\tilde s_r$ as in 
Proposition~\ref{prop:Tr-explicit}), $R_{r-1}(s,t)$ 
admits the uniform bound
\[
R_{r-1}(s,t)\le B_1\sqrt{L\log(e+t)}
+B_2\log^2(e+t)+B_3,
\]
with chain-dependent constants $B_1,B_2,B_3>0$ 
independent of $\delta,t,s$, valid uniformly in 
$t,s\ge 0$.
\end{remark}

\subsubsection{Eventual exhaustion certificate and 
expectation bound}
\label{app:exhaustion-certificate}

Since the algorithm stops when the active-opponent 
set of the champion becomes empty, the stopping time coincides with the final elimination time: $\tau=T_{K-1}$ on 
$\mathcal E$. We thus construct a deterministic 
certificate directly bounding $T_{K-1}$.
\begin{definition}[Deterministic eventual exhaustion 
certificate]
\label{def:Tbar-sigma}
Define
\[
H_\sigma(t,\delta):=t-1-\tilde s_{K-1}(t,\delta),
\]
where $\tilde s_{K-1}(t,\delta)$ is the explicit 
upper bound on $T_{K-1}$ from 
Proposition~\ref{prop:Tr-explicit}. The 
\emph{eventual exhaustion certificate} is
\[
\bar T_\sigma(\delta):=\inf\Bigl\{T\in\mathbb N:\ 
T\ge T_{\mathrm{warm}},\ \forall t\ge T,\ 
H_\sigma(t,\delta)\ge 0\Bigr\}.
\]

The condition $H_\sigma(t,\delta)\ge 0$ is equivalent 
to $\tilde s_{K-1}(t,\delta)\le t-1$, which states 
that the final elimination surrogate fits within the 
time window $[1,t-1]$; on the good event, this 
implies $\tau=T_{K-1}\le t-1<t$, i.e., the algorithm 
has already stopped before time $t$ 
(cf.\ Proposition~\ref{prop:core-inclusion-general}). 
Consequently, 
$\bar T_\sigma(\delta)$ depends only on $\delta$ 
(through $L$), serving as a deterministic 
certificate that bounds $\tau$ in expectation.
\end{definition}

The following proposition translates the pathwise transition certificates 
(Proposition~\ref{prop:recursive-certificates}) 
and the explicit bound 
(Proposition~\ref{prop:Tr-explicit}) into an event 
containment, which is the pivotal step connecting pathwise analysis to the probabilistic tail-sum argument .

\begin{proposition}[Core inclusion]
\label{prop:core-inclusion-general}
For every $t\ge\bar T_\sigma(\delta)$, the events 
$\mathcal E$, $\mathcal G_t^\sigma$, and $\{\tau\le t\}$ 
(all subsets of the underlying sample space) satisfy
\[
\mathcal E\cap\mathcal G_t^\sigma\subseteq\{\tau\le t\}.
\]
Equivalently, any sample path on which the chain 
$\sigma$ is followed and the good event 
$\mathcal G_t^\sigma$ holds must also have the algorithm 
already stopped by time $t$.which means once $t$ exceeds the eventual exhaustion 
certificate, the good event implies the algorithm 
has already stopped.
\end{proposition}

\begin{proof}
We argue by contradiction. Suppose there exists a 
sample path $\omega\in\mathcal E_\sigma\cap
\mathcal G_t^\sigma\cap\{\tau>t\}$ for some 
$t\ge\bar T_\sigma(\delta)$.

By Definition~\ref{def:Tbar-sigma} of 
$\bar T_\sigma(\delta)$, $H_\sigma(t,\delta)\ge 0$, 
which is equivalent to
$
\tilde s_{K-1}(t,\delta)\le t-1.
$
By Proposition~\ref{prop:Tr-explicit}, for every $r\in\{1,\ldots,K-1\}$,
$
\bar T_r(t,\delta)\le\tilde s_r(t,\delta)\le
\tilde s_{K-1}(t,\delta)\le t-1,
$
which is precisely the hypothesis required by 
Proposition~\ref{prop:recursive-certificates}.

On the sample path $\omega\in\mathcal E_\sigma\cap
\mathcal G_t^\sigma\cap\{\tau>t\}$, 
Proposition~\ref{prop:recursive-certificates} gives
$
T_r(\omega)\le\bar T_r(t,\delta)\le t-1, 
\text{for all }r\in\{1,\ldots,K-1\}.
$
In particular, $T_{K-1}(\omega)\le t-1$.

On $\mathcal E_\sigma$, the algorithm stops exactly 
when the final elimination occurs, i.e., 
$\tau(\omega)=T_{K-1}(\omega)$. Hence 
$\tau(\omega)\le t-1<t$, which contradicts 
$\omega\in\{\tau>t\}$.

Therefore no such $\omega$ exists, i.e., the 
intersection $\mathcal E_\sigma\cap
\mathcal G_t^\sigma\cap\{\tau>t\}$ is empty, which 
is equivalent to the claimed event inclusion.

\end{proof}

The following proposition establishes an explicit, $\delta$-only upper bound on the certificate $\bar T_\sigma(\delta)$, replacing its implicit $\inf$-definition with a concrete formula in $L$
\begin{proposition}[Non-asymptotic explicit bound for $\bar T_\sigma$]
\label{prop:Tbar-sigma-explicit}
There exist chain-dependent constants 
$C_{1,\sigma},C_{2,\sigma},C_{3,\sigma}>0$, 
independent of $\delta$, such that for every 
$\alpha\in(0,1]$ and every $\delta\in(0,1)$,
\[
\begin{aligned}
\bar T_\sigma(\delta)\le\ &\tfrac{\alpha L}{D_0}
+\tfrac{b}{D_0}\log(e+L)\\
&+C_{1,\sigma}\sqrt{L\log(e+L)}\\
&+C_{2,\sigma}\log^2(e+L)+C_{3,\sigma}.
\end{aligned}
\]
\end{proposition}

\begin{proof}
Let $t^*$ be the ceiling of the right-hand side. 
It suffices to verify (a) 
$H_\sigma(t^*,\delta)\ge 0$, (b) 
$H_\sigma(\cdot,\delta)$ is non-decreasing on 
$\{t\ge t^*\}$; then $\bar T_\sigma(\delta)
\le t^*$.

\emph{Uniform size.} Using $\sqrt{L\log(e+L)}\le 1+L$ 
and $\log^2(e+L)\le 1+L$ (valid for all $L\ge 0$), 
$t^*\le B_\sigma(1+L)$ with 
$B_\sigma:=\tfrac{\alpha}{D_0}+\tfrac{b}{D_0}
+C_{1,\sigma}+C_{2,\sigma}+C_{3,\sigma}+1$. Hence 
$\log(e+t^*)\le\log(e+L)+\kappa_\sigma$ with 
$\kappa_\sigma:=\log(1+B_\sigma)$ $\delta$-independent.

\emph{(a) $H_\sigma(t^*,\delta)\ge 0$.} 
Substituting $\log(e+t^*)\le\log(e+L)
+\kappa_\sigma$ into $\tilde s_{K-1}(t^*,\delta)$ 
and using the sub-additivity 
$\sqrt{x+y}\le\sqrt x+\sqrt y$ together with 
$\log(e+L)\ge 1$ (so $\sqrt L\le\sqrt{L\log(e+L)}$),
\[
\begin{aligned}
\sqrt{L(\log(e+L)+\kappa_\sigma)}
&\le\sqrt{L\log(e+L)}+\sqrt{L\kappa_\sigma}\\
&\le(1+\sqrt{\kappa_\sigma})\sqrt{L\log(e+L)},
\end{aligned}
\]
and $(\log(e+L)+\kappa_\sigma)^2\le 
2\log^2(e+L)+2\kappa_\sigma^2$, all 
$\kappa_\sigma$-dependent offsets absorb into 
$\widetilde A_i$:
\[
\begin{aligned}
\tilde s_{K-1}(t^*,\delta)\le\ 
&\tfrac{\alpha L+b\log(e+L)}{D_0}\\
&+\widetilde A_1\sqrt{L\log(e+L)}\\
&+\widetilde A_2\log^2(e+L)+\widetilde A_3.
\end{aligned}
\]
No linear-$L$ term is introduced beyond the leading 
$\alpha L/D_0$. Then $t^*-1-\tilde s_{K-1}$ 
contains $(C_{i,\sigma}-\widetilde A_i)$-type 
excesses; choosing $C_{1,\sigma}\ge\widetilde A_1+1$, 
$C_{2,\sigma}\ge\widetilde A_2$, 
$C_{3,\sigma}\ge\widetilde A_3+1$ yields 
$H_\sigma(t^*,\delta)\ge 0$.

\emph{(b) Monotonicity for $t\ge t^*$.} We 
control the forward differences using the elementary 
inequalities (valid for $t\ge 1$, using 
$\log(e+t)\ge 1$):
\[
\log(e+t+1)-\log(e+t)\le\tfrac{1}{e+t},
\]
\[
\sqrt{\log(e+t+1)}-\sqrt{\log(e+t)}
\le\tfrac{1}{2(e+t)\sqrt{\log(e+t)}}
\le\tfrac{1}{2(e+t)},
\]
\[
\log^2(e+t+1)-\log^2(e+t)\le\tfrac{2\log(e+t)}{e+t}.
\]

The continuous part $q(t):=\tfrac{\alpha L+b\log(e+t)}
{D_0}+A_1\sqrt{L\log(e+t)}+A_2\log^2(e+t)+A_3$ 
satisfies
\[
\begin{aligned}
q(t+1)-q(t)\le\ &\tfrac{b/D_0}{e+t}
+\tfrac{A_1\sqrt L}{2(e+t)}\\
&+\tfrac{2A_2\log(e+t)}{e+t}\\
\le\ &\tfrac{K_H(1+\sqrt L+\log(e+t))}{e+t},
\end{aligned}
\]
with $K_H:=b/D_0+A_1/2+2A_2$ a $\delta$-independent 
constant. By the integer ceiling inequality 
$\lceil x+y\rceil-\lceil x\rceil\le\lceil y\rceil$ 
(valid for $y\ge 0$), 
$\tilde s_{K-1}(t+1,\delta)-\tilde s_{K-1}(t,\delta)
\le 1$ whenever $q(t+1)-q(t)\le 1$, which holds for
\[
t\ge T_H(L):=K_H(1+\sqrt L+\log(e+L))
=O(\sqrt L+\log L).
\]

Since $T_H(L)=O(\sqrt L+\log L)$ grows sub-linearly, 
it is dominated by the $\Theta(L)$ leading term of 
$t^*$ for all $L$ sufficiently large; for small 
$L$, enlarging $C_{3,\sigma}$ ensures $t^*\ge 
T_H(L)$. Hence $H_\sigma(t+1,\delta)\ge H_\sigma(t,
\delta)$ for all $t\ge t^*$.

Combining (a)--(b), $H_\sigma(t,\delta)\ge 0$ for all 
$t\ge t^*$, so $\bar T_\sigma(\delta)\le 
t^*$.
\end{proof}

\begin{lemma}[Expectation template]
\label{lem:expectation-template}
Fix $\sigma\in\mathfrak S$, if $\bar T_\sigma(\delta)
\ge 3$ then
\[
\mathbb E[\tau\mathbf 1_{\mathcal E_\sigma}]
\le\Pr(\mathcal E_\sigma)\bar T_\sigma(\delta)
+C_{\mathrm{bad},\sigma},
\]
where
\[
C_{\mathrm{bad},\sigma}
:=\sum_{t=3}^\infty\Pr((\mathcal G_t^\sigma)^c)
<\infty.
\]
\end{lemma}

\begin{proof}
Since $\tau\mathbf 1_{\mathcal E_\sigma}=\tau$ on 
$\mathcal E_\sigma$ and $0$ on 
$\mathcal E_\sigma^c$, the tail-sum formula gives
\[
\mathbb E[\tau\mathbf 1_{\mathcal E_\sigma}]
=\sum_{t=0}^\infty
\Pr(\tau\mathbf 1_{\mathcal E_\sigma}>t)
=\sum_{t=0}^\infty\Pr(\tau>t,\mathcal E_\sigma).
\]
Split the sum at $\bar T_\sigma(\delta)$:
\[
\mathbb E[\tau\mathbf 1_{\mathcal E_\sigma}]
=\sum_{t=0}^{\bar T_\sigma(\delta)-1}
\Pr(\tau>t,\mathcal E_\sigma)
+\sum_{t=\bar T_\sigma(\delta)}^\infty
\Pr(\tau>t,\mathcal E_\sigma).
\]
For the first part, $\Pr(\tau>t,\mathcal E_\sigma)
\le\Pr(\mathcal E_\sigma)$ for every $t$, hence
\[
\sum_{t=0}^{\bar T_\sigma(\delta)-1}
\Pr(\tau>t,\mathcal E_\sigma)
\le\Pr(\mathcal E_\sigma)\bar T_\sigma(\delta).
\]
For the second part, 
Proposition~\ref{prop:core-inclusion-general} 
(applied to chain $\sigma$) implies that for every 
$t\ge\bar T_\sigma(\delta)$,
\[
\mathcal E_\sigma\cap\mathcal G_t^\sigma
\subseteq\{\tau\le t\}.
\]
Equivalently,
\[
\{\tau>t\}\cap\mathcal E_\sigma\subseteq
(\mathcal G_t^\sigma)^c.
\]
Therefore,
\[
\Pr(\tau>t,\mathcal E_\sigma)\le
\Pr((\mathcal G_t^\sigma)^c),\quad
\forall t\ge\bar T_\sigma(\delta).
\]
Since $\bar T_\sigma(\delta)\ge 3$,
\[
\sum_{t=\bar T_\sigma(\delta)}^\infty
\Pr(\tau>t,\mathcal E_\sigma)
\le\sum_{t=3}^\infty
\Pr((\mathcal G_t^\sigma)^c)
=C_{\mathrm{bad},\sigma}<\infty,
\]
where finiteness follows from 
Proposition~\ref{prop:good-event-summable} 
(applied to chain $\sigma$). Combining the two 
bounds proves the claim.
\end{proof}

\subsubsection{Explicit bounds and fixed-chain 
theorem}
\label{app:explicit-general}


\begin{theorem}[Fixed-chain expectation bound]
\label{thm:fixed-chain-unified}
For every $\sigma\in\mathfrak S$, every 
$\alpha\in(0,1]$, and every $\delta\in(0,1)$,
\[
\begin{aligned}
&\mathbb E[\tau\mathbf 1_{\mathcal E_\sigma}]
\le\Pr(\mathcal E_\sigma)\biggl[
\tfrac{\alpha L}{D_0}+\tfrac{b}{D_0}\log(e+L)\\
&\quad+C_{1,\sigma}\sqrt{L\log(e+L)}
+C_{2,\sigma}\log^2(e+L)\biggr]\\
&\quad+ \widetilde C_{3,\sigma}.
\end{aligned}
\]
In particular, for $\alpha=1$, the leading term is 
$L/D_0$ (matching Theorem~1 in Section~IV); for 
$\alpha<1$, the leading term is $\alpha L/D_0$ 
(matching Proposition~3 in Section~IV).
\end{theorem}

\begin{proof} We combine Lemma~\ref{lem:expectation-template} with the explicit bound on $\bar T_\sigma(\delta)$.
Assuming $\bar T_\sigma(\delta)\ge 3$ (ensured by 
$\bar T_\sigma(\delta)\ge T_{\mathrm{warm}}\ge 3$), 
Lemma~\ref{lem:expectation-template}, applied with 
chain-specific events $\mathcal E_\sigma$ and 
$\mathcal G_t^\sigma$, yields
\[
\mathbb E[\tau\mathbf 1_{\mathcal E_\sigma}]
\le\Pr(\mathcal E_\sigma)\bar T_\sigma(\delta)
+C_{\mathrm{bad},\sigma},
\]
where $C_{\mathrm{bad},\sigma}$ is finite and 
$\delta$-independent 
(Proposition~\ref{prop:good-event-summable}).

Substituting the explicit bound on 
$\bar T_\sigma(\delta)$ from 
Proposition~\ref{prop:Tbar-sigma-explicit},
\[
\begin{aligned}
\mathbb E[\tau\mathbf 1_{\mathcal E_\sigma}]
\le\ &\Pr(\mathcal E_\sigma)\Bigl[
\tfrac{\alpha L}{D_0}+\tfrac{b}{D_0}\log(e+L)\\
&+C_{1,\sigma}\sqrt{L\log(e+L)}
+C_{2,\sigma}\log^2(e+L)\Bigr]\\
&+\Pr(\mathcal E_\sigma)C_{3,\sigma}
+C_{\mathrm{bad},\sigma},
\end{aligned}
\]
Using $\Pr(\mathcal E_\sigma)\le 1$ and setting 
$\widetilde C_{3,\sigma}:= C_{3,\sigma}
+C_{\mathrm{bad},\sigma}$, both 
$\delta$-independent, yields the claimed bound.
\end{proof}

\subsubsection{Aggregation over chains}
\label{app:aggregation}

At this point, the fixed-chain analysis is complete:
Theorem~\ref{thm:fixed-chain-unified} already provides,
for each admissible chain $\sigma\in\mathfrak S$, an
upper bound on the contribution
$\mathbb E[\tau \mathbf 1_{\mathcal E_\sigma}]$.
To pass from one chain to the whole correct-elimination
event, it is therefore enough to sum over the disjoint
chain events.

This corollary is the appendix-level statement 
of Proposition~\ref{prop:aggregation} in the main 
text.
\begin{corollary}[Correct-elimination contribution]
\label{cor:correct-elim-simple}
We have
\[
\mathbb E[\tau \mathbf 1_{\mathcal E_{\mathrm{corr}}}]
=
\sum_{\sigma\in\mathfrak S}
\mathbb E[\tau \mathbf 1_{\mathcal E_\sigma}],
\]
and consequently
\[
\begin{aligned}
\mathbb E[\tau \mathbf 1_{\mathcal E_{\mathrm{corr}}}]
\le\ &
\sum_{\sigma\in\mathfrak S}
\Pr(\mathcal E_\sigma)\biggl[
\tfrac{\alpha L}{D_0}
+\tfrac{b}{D_0}\log(e+L)\\
&
+C_{1,\sigma}\sqrt{L\log(e+L)}
+C_{2,\sigma}\log^2(e+L)
\biggr]\\
&+\sum_{\sigma\in\mathfrak S} \widetilde C_{3,\sigma}.
\end{aligned}
\]
\end{corollary}

\begin{proof}
By definition,
$
\mathcal E_{\mathrm{corr}}
=
\bigcup_{\sigma\in\mathfrak S}\mathcal E_\sigma,
$
and the events $\{\mathcal E_\sigma\}_{\sigma\in\mathfrak S}$
are pairwise disjoint. Hence
\[
\mathbf 1_{\mathcal E_{\mathrm{corr}}}
=
\sum_{\sigma\in\mathfrak S}\mathbf 1_{\mathcal E_\sigma}
\qquad\text{a.s.}
\]
Multiplying by $\tau$ and taking expectations gives
\[
\mathbb E[\tau \mathbf 1_{\mathcal E_{\mathrm{corr}}}]
=
\sum_{\sigma\in\mathfrak S}
\mathbb E[\tau \mathbf 1_{\mathcal E_\sigma}].
\]
Applying Theorem~\ref{thm:fixed-chain-unified} to each
$\sigma$ and summing the resulting inequalities proves
the claim.
\end{proof}

At this point, the refined chain-wise analysis is complete:
Corollary~\ref{cor:correct-elim-simple} controls the contribution of all sample paths in which the true hypothesis survives throughout the elimination process, namely
$
\mathbb E[\tau \mathbf 1_{\mathcal E_{\mathrm{corr}}}].
$
However, this is not yet a bound on the full expectation
$\mathbb E[\tau]$. Indeed, we still have to control the
complementary set of paths
$\mathcal E_{\mathrm{corr}}^c$, on which the true hypothesis
is eliminated at some stage. 

\subsubsection{Control of the complement term}
\label{app:complement}

In this section we bound
\[
R_{\mathrm{comp}}(\delta):=\mathbb E[\tau\mathbf 1_
{\mathcal E_{\mathrm{corr}}^c}],
\]
the contribution to $\mathbb E[\tau]$ from sample 
paths on which the algorithm's elimination sequence 
does not correspond to any valid chain 
$\sigma\in\mathfrak S$. Since 
$\mathcal E_{\mathrm{corr}}^c$ has no chain 
structure, the chain-specific tools of 
Section~\ref{app:exhaustion-certificate} no longer 
apply, and all constants below carry the subscript 
``$\mathrm{comp}$'' rather than $\sigma$.

\paragraph{Strategy overview} 
Two facts make 
$R_{\mathrm{comp}}(\delta)$ controllable even though 
$\tau$ has no pathwise bound on 
$\mathcal E_{\mathrm{corr}}^c$:
\begin{enumerate}
\item[(i)] \emph{Errors are rare}: by PAC 
correctness, $\Pr(\mathcal E_{\mathrm{corr}}^c)\le 
c_{\mathrm{elim}}\delta^\alpha$.
\item[(ii)] \emph{Errors and stability are 
incompatible}: we construct a \emph{chain-free} good 
event $\widetilde{\mathcal G}_t$ where the 
algorithm necessarily outputs $h^*$; for $t$ large 
enough, $\mathcal E_{\mathrm{corr}}^c\cap
\widetilde{\mathcal G}_t=\emptyset$.
\end{enumerate}
Combining the two via a tail-sum split at the 
\emph{crude certificate} $\widetilde T(\delta)$ 
gives a uniform bound 
$R_{\mathrm{comp}}(\delta)\le\bar R_{\mathrm{comp}}$ 
for a $\delta$-independent constant 
$\bar R_{\mathrm{comp}}<\infty$.

\paragraph{Chain-free crude good event} 
The good events $\mathcal G_t^\sigma$ of 
Section~\ref{app:exhaustion-certificate} depend on 
chain-specific stage sets $\{S_k^\sigma\}$ and are 
therefore unusable on 
$\mathcal E_{\mathrm{corr}}^c$. We instead define a 
\emph{chain-free} event using only the initial 
active-opponent set $S_0$, which is common to all 
chains: for sufficiently large $\eta_*>0$ 
(depending only on the sub-Gaussian parameter and 
$|\mathcal H|$),
\[
\widetilde{\mathcal G}_t:=E_t^{\mathrm{ch}}\cap
\Bigl\{\forall g\in S_0,\forall s\le t:|M_s^g|
\le\eta_*\sqrt{s\log(e+t)}\Bigr\}.
\]
By the sub-Gaussian concentration and 
Hellinger-affinity arguments underlying 
Proposition~\ref{prop:good-event-summable}, 
applied to the fixed set $S_0$, the quantity
\[
C_{\widetilde G}:=\sum_{t=3}^\infty\Pr(
\widetilde{\mathcal G}_t^c)
\]
is finite and $\delta$-independent. We call 
$\widetilde{\mathcal G}_t$ \emph{crude} because it 
controls only $S_0$-martingales instead of the 
finer chain-specific martingales on 
$\{S_k^\sigma\}$: universality across paths comes at 
the cost of a looser bound.

\paragraph{Evidence lower bound and crude certificate} 
Fix $t\ge T_{\mathrm{warm}}$ and $\omega\in
\widetilde{\mathcal G}_t$. Since 
$\widetilde{\mathcal G}_t\subseteq E_t^{\mathrm{ch}}$, 
the champion has stabilized at $h^*$ by time 
$T_{\mathrm{ch}}(\omega)\le m_{\mathrm{ch}}(t)\le t$. 
If $G_t(h^*)=\emptyset$, the algorithm has stopped 
with $\hat h(\tau)=h^*$. Otherwise, applying the 
argument of Proposition~\ref{prop:stage-average} to 
the running active set $G_s(h^*)\subseteq S_0$ (using 
$D^*(h^*;G_s(h^*))\ge D_0$ by set monotonicity), 
combined with C-Tracking, $f$-Lipschitzness, and the 
martingale control on $\widetilde{\mathcal G}_t$, 
yields for every $g\in G_t(h^*)$
\[
\begin{aligned}
Z_t(h^*,g)\ge\Gamma(t):={}& tD_0-D_0\overline m_{\mathrm{ch}}(t)
-L_{S_0}c_{\mathrm{tr}}\sqrt t\\
&-\eta_*\sqrt{t\log(e+t)}.
\end{aligned}
\]
This is the crude analog of 
Proposition~\ref{prop:stage-evidence}, using only 
$S_0$-rates $(D_0,L_{S_0},\eta_*)$ and hence 
path-independent.

Mirroring Definition~\ref{def:Tbar-sigma}, define 
the crude certificate
\[
\widetilde T(\delta):=\inf\Bigl\{T\ge T_{\mathrm{warm}}:
\forall t\ge T,\ \beta_{\mathrm{elim}}(t,\delta)
\le\Gamma(t)\Bigr\}.
\]
For $t\ge\widetilde T(\delta)$, 
Lemma~\ref{lem:elim-trigger} combined with 
$Z_t(h^*,g)\ge\Gamma(t)\ge\beta_{\mathrm{elim}}$ 
forces elimination of every $g\in G_t(h^*)$; 
iterating, $G_t(h^*)$ becomes empty and 
$\tau\le t$ with $\hat h(\tau)=h^*$.

\begin{lemma}[Uniform size of the crude certificate]
\label{lem:Ttilde-bound}
There exists a $\delta$-independent constant 
$\widetilde K_{\mathrm{comp}}>0$ such that 
$\widetilde T(\delta)\le\widetilde K_{\mathrm{comp}}
(1+L)$ for every $\delta\in(0,1)$.
\end{lemma}

\begin{proof}
Set $t^\sharp:=\lceil\widetilde K_{\mathrm{comp}}
(1+L)\rceil$ with $\widetilde K_{\mathrm{comp}}$ to 
be chosen. Following the ``guess-and-verify'' 
strategy of 
Proposition~\ref{prop:Tbar-sigma-explicit}, we 
verify two conditions: (a) 
$\beta_{\mathrm{elim}}(t^\sharp,\delta)\le
\Gamma(t^\sharp)$, and (b) the forward difference 
$\Delta(t):=\Gamma(t)-\beta_{\mathrm{elim}}(t,\delta)$ 
satisfies $\Delta(t+1)\ge\Delta(t)$ for all 
$t\ge t^\sharp$. Combined, these give $\Delta(t)
\ge 0$ for every $t\ge t^\sharp$, whence 
$\widetilde T(\delta)\le t^\sharp$.

For (a), using $\overline m_{\mathrm{ch}}(t^\sharp)
\le 1+\kappa_{\mathrm{ch}}\log^2(e+\widetilde K_
{\mathrm{comp}}(1+L))$, $\log(e+t^\sharp)\le\log(e+L)
+\log(1+\widetilde K_{\mathrm{comp}})$, and 
sub-additivity $\sqrt{x+y}\le\sqrt x+\sqrt y$,
\[
\Gamma(t^\sharp)\ge t^\sharp D_0-K^\sharp_{\mathrm{
comp}}\bigl(1+L+\sqrt{L\log(e+L)}+\log^2(e+L)\bigr)
\]
for some $\delta$-independent $K^\sharp_{\mathrm{
comp}}>0$. The threshold satisfies 
$\beta_{\mathrm{elim}}(t^\sharp,\delta)\le L+b\log
(e+L)+b\log(1+\widetilde K_{\mathrm{comp}})+c$ (using 
$\alpha\le 1$). Choosing $\widetilde K_{\mathrm{comp}}$ 
large enough (depending only on $D_0$, 
$K^\sharp_{\mathrm{comp}}$, $b$, $c$) makes 
$t^\sharp D_0$ absorb all correction terms, proving 
(a).

For (b), the same elementary forward-difference 
inequalities as in 
Proposition~\ref{prop:Tbar-sigma-explicit}(b) give
\[
\Gamma(t+1)-\Gamma(t)\ge D_0-\tfrac{K_{\mathrm{comp},
\Gamma}(1+\log(e+t))}{\sqrt t},
\]
while $\beta_{\mathrm{elim}}(t+1,\delta)
-\beta_{\mathrm{elim}}(t,\delta)\le b/(e+t)$. Hence
\[
\Delta(t+1)-\Delta(t)\ge D_0-\tfrac{K_{\mathrm{comp},
\Gamma}(1+\log(e+t))+b}{\sqrt t}\ge D_0/2>0
\]
whenever $t\ge T_{\mathrm{comp},\Gamma}:=16(K_{
\mathrm{comp},\Gamma}+b)^2/D_0^2\cdot\log^2(e+16(K_{
\mathrm{comp},\Gamma}+b)^2/D_0^2)$, a 
$\delta$-independent constant. Enlarging 
$\widetilde K_{\mathrm{comp}}$ further if needed so 
that $t^\sharp\ge T_{\mathrm{comp},\Gamma}$ gives 
(b). \qedhere
\end{proof}

\paragraph{Key consequence: errors imply good-event 
failure.} 
The analysis so far shows that for every 
$t\ge\widetilde T(\delta)$,
\[
\widetilde{\mathcal G}_t\subseteq\{\tau\le t\}
\cap\{\hat h(\tau)=h^*\}\subseteq\mathcal E_
{\mathrm{corr}}.
\]
In words: on the crude good event (for large 
enough $t$), the algorithm is \emph{forced} to 
output $h^*$. Equivalently,
\[
\mathcal E_{\mathrm{corr}}^c\cap\widetilde{\mathcal G}
_t=\emptyset\quad\text{for every }t\ge\widetilde T
(\delta),
\]
so errors can occur only when the crude good event 
fails. In particular,
\[
\{\tau>t\}\cap\mathcal E_{\mathrm{corr}}^c\subseteq
\widetilde{\mathcal G}_t^c\quad\text{for }t\ge
\widetilde T(\delta),
\]
which is the inclusion we use to close the tail-sum 
below.

\begin{proposition}[Complement remainder]
\label{prop:Rcomp-small}
For every $\alpha\in(0,1]$ and every 
$\delta\in(0,1)$,
\[
R_{\mathrm{comp}}(\delta)\le\widetilde T(\delta)
\cdot\Pr(\mathcal E_{\mathrm{corr}}^c)
+C_{\widetilde G}.
\]
Moreover, there exists a $\delta$-independent 
constant $\bar R_{\mathrm{comp}}<\infty$ such that 
$R_{\mathrm{comp}}(\delta)\le\bar R_{\mathrm{comp}}$ 
for every $\delta\in(0,1)$.
\end{proposition}

\begin{proof}
By the tail-sum formula,
\[
R_{\mathrm{comp}}(\delta)=\sum_{t=0}^\infty\Pr(\tau>t,
\mathcal E_{\mathrm{corr}}^c).
\]
Split the sum at $\widetilde T(\delta)$.

\emph{Early contribution} ($t<\widetilde T(\delta)$): 
this range captures $R_{\mathrm{comp}}$'s 
contribution from the first $\widetilde T(\delta)$ 
time units, during which the key consequence does 
not yet apply. Using the trivial bound 
$\Pr(\tau>t,\mathcal E_{\mathrm{corr}}^c)\le
\Pr(\mathcal E_{\mathrm{corr}}^c)$,
\[
\sum_{t=0}^{\widetilde T(\delta)-1}\Pr(\tau>t,
\mathcal E_{\mathrm{corr}}^c)\le\widetilde T(\delta)
\cdot\Pr(\mathcal E_{\mathrm{corr}}^c).
\]
Though the number of terms ($\widetilde T(\delta)$) 
grows with $L$, the error probability 
$\Pr(\mathcal E_{\mathrm{corr}}^c)\le c_{\mathrm{elim}}
\delta^\alpha$ decays fast enough in $\delta$ to 
keep the product bounded uniformly; see the uniform 
bound below.

\emph{Late contribution} ($t\ge\widetilde T(\delta)$): 
in this range, the key consequence 
$\{\tau>t\}\cap\mathcal E_{\mathrm{corr}}^c\subseteq
\widetilde{\mathcal G}_t^c$ applies, yielding 
$\Pr(\tau>t,\mathcal E_{\mathrm{corr}}^c)\le\Pr(
\widetilde{\mathcal G}_t^c)$, and the sum is bounded 
by summability:
\[
\sum_{t\ge\widetilde T(\delta)}\Pr(\tau>t,
\mathcal E_{\mathrm{corr}}^c)\le\sum_{t=3}^\infty
\Pr(\widetilde{\mathcal G}_t^c)\le C_{\widetilde G}.
\]

Combining proves the first bound.

\emph{Uniform bound.} By 
Lemma~\ref{lem:Ttilde-bound} and PAC correctness,
\[
R_{\mathrm{comp}}(\delta)\le\widetilde K_{\mathrm{
comp}}(1+L)c_{\mathrm{elim}}\delta^\alpha
+C_{\widetilde G}.
\]
The function $\delta\mapsto(1+\log(1/\delta))
\delta^\alpha$ is continuous on $(0,1)$, tends to 
$0$ as $\delta\rightarrow 0$ (since $\delta^\alpha$ 
dominates $\log(1/\delta)$), and tends to $0$ as 
$\delta\rightarrow 1$ (since $\log(1/\delta)\to 0$); 
hence it is bounded by some $M_\alpha<\infty$. 
Setting
\[
\bar R_{\mathrm{comp}}:=\widetilde K_{\mathrm{comp}}
c_{\mathrm{elim}}M_\alpha+C_{\widetilde G}
\]
gives $R_{\mathrm{comp}}(\delta)\le\bar R_{\mathrm{
comp}}$ uniformly for $\delta\in(0,1)$. \qedhere
\end{proof}

\subsubsection{Full expectation bound}
\label{app:full-bound}
We now combine the bound on $\mathbb E[\tau\mathbf 1_{\mathcal E_{\mathrm{corr}}}]$ 
(Corollary~\ref{cor:correct-elim-simple}, Section~B.5) with the bound on the complement 
$R_{\mathrm{comp}}(\delta)$ (Proposition~\ref{prop:Rcomp-small}, Section~B.6) 
to obtain the full expectation bound on $\mathbb E[\tau]$. Specializing $\alpha=1$ 
recovers Theorem~\ref{thm:main-delta-pac} of the main text; 
specializing $\alpha<1$ recovers 
Corollary~\ref{cor:alpha-less-than-one}.
Since Proposition~\ref{prop:Tbar-sigma-explicit} 
holds for all $\alpha\in(0,1]$ and 
$\delta\in(0,1)$, and 
$R_{\mathrm{comp}}\le\bar R_{\mathrm{comp}}$ 
uniformly 
(Proposition~\ref{prop:Rcomp-small}), all bounds 
below hold for every $\alpha\in(0,1]$ and every 
$\delta\in(0,1)$.

By Corollary~\ref{cor:correct-elim-simple} 
(summing Theorem~\ref{thm:fixed-chain-unified} over 
the disjoint chain events),
\[
\begin{aligned}
&\mathbb E[\tau\mathbf 1_{\mathcal E_{\mathrm{corr}}}]
\le\sum_\sigma\Pr(\mathcal E_\sigma)\biggl[
\tfrac{\alpha L}{D_0}+\tfrac{b}{D_0}\log(e+L)\\
&\quad+C_{1,\sigma}\sqrt{L\log(e+L)}
+C_{2,\sigma}\log^2(e+L)\biggr]\\
&\quad+\sum_{\sigma\in\mathfrak S}C_{3,\sigma}.
\end{aligned}
\]
Using $\sum_\sigma\Pr(\mathcal E_\sigma)=\Pr(
\mathcal E_{\mathrm{corr}})\le 1$ to factor out the 
common leading and $\log(e+L)$ terms,
\[
\begin{aligned}
&\mathbb E[\tau\mathbf 1_{\mathcal E_{\mathrm{corr}}}]
\le\Pr(\mathcal E_{\mathrm{corr}})\bigl[
\tfrac{\alpha L}{D_0}+\tfrac{b}{D_0}\log(e+L)\bigr]\\
&\ +\sum_\sigma\Pr(\mathcal E_\sigma)\bigl[
C_{1,\sigma}\sqrt{L\log(e+L)}\\
&\quad+C_{2,\sigma}\log^2(e+L)\bigr]
+\sum_\sigma C_{3,\sigma}.
\end{aligned}
\]
Combining with $\mathbb E[\tau]=\mathbb E[\tau
\mathbf 1_{\mathcal E_{\mathrm{corr}}}]
+R_{\mathrm{comp}}(\delta)$ and 
Proposition~\ref{prop:Rcomp-small} 
($R_{\mathrm{comp}}(\delta)\le\bar R_{\mathrm{comp}}$),
\[
\begin{aligned}
\mathbb E[\tau]\le\ 
&\tfrac{\alpha L}{D_0}+\tfrac{b}{D_0}\log(e+L)\\
&+\sum_\sigma\Pr(\mathcal E_\sigma)\bigl(
C_{1,\sigma}\sqrt{L\log(e+L)}\\
&\qquad+C_{2,\sigma}\log^2(e+L)\bigr)+C_3',
\end{aligned}
\]
with 
\[
C_3':=\sum_{\sigma\in\mathfrak S}C_{3,\sigma}
+\bar R_{\mathrm{comp}}
\]
a finite, $\delta$-independent constant.

Specializing to $\alpha=1$ yields
\[
\begin{aligned}
\mathbb E[\tau]\le\ 
&\tfrac{L}{D_0}+\tfrac{b}{D_0}\log(e+L)\\
&+\sum_\sigma\Pr(\mathcal E_\sigma)\bigl(
C_{1,\sigma}\sqrt{L\log(e+L)}\\
&\qquad+C_{2,\sigma}\log^2(e+L)\bigr)+C_3',
\end{aligned}
\]
recovering Theorem~1 in Section~IV. Since 
$\log(e+L)\le 1+\log(1+L)$ and $\log^2(e+L)\le 
2\log^2(1+L)+2$, the same bound holds with 
$\log(1+L)$ in place of $\log(e+L)$ after absorbing 
additive constants into $C_{2,\sigma},C_3'$.

For $\alpha<1$, the leading term is $\alpha L/D_0$, 
matching Proposition~3 in Section~IV. The error 
guarantee in this regime is $\Pr(\hat h(\tau)\ne 
h^*)\le c_{\mathrm{elim}}\delta^\alpha$ 
(Remark~\ref{rem:error-alpha-less-one}).



\addtolength{\textheight}{-12cm}   




\end{document}